\documentclass{cmslatex}
\usepackage[paperwidth=7in, paperheight=10in, margin=.875in]{geometry}
\usepackage[backref,colorlinks,linkcolor=red,anchorcolor=green,citecolor=blue]{hyperref}
\usepackage{amsfonts,amssymb}
\usepackage{amsmath}
\usepackage{graphicx}
\usepackage{cite}
\usepackage{enumerate}

\usepackage{algorithm}
\usepackage{algorithmic}

\newcommand{\MYCOMMENT}[1]{\textcolor{gray}{// #1}}

\usepackage{subcaption}

\usepackage{mathtools}
\usepackage{multirow}
\usepackage{xcolor}
\usepackage{microtype}
\usepackage{graphicx}
\usepackage{booktabs}

\renewcommand{\theequation}{\arabic{section}.\arabic{equation}}

\allowdisplaybreaks
\begin{document}

\title{Riemannian Denoising Diffusion Probabilistic Models}


\author{
Zichen Liu\thanks{\textbf{Equal contribution.} Center for Data Science, Peking University, Beijing 100871, P.R. China (zcliu@stu.pku.edu.cn).}
\and Wei Zhang\thanks{\textbf{Equal contribution.} Zuse Institute Berlin, Takustrasse 7, Berlin 14195, Germany (wei.zhang@fu-berlin.de).}
\and Christof Schütte\thanks{Institute of Mathematics, Freie Universität Berlin and Zuse Institute Berlin, Takustrasse 7, Berlin 14195, Germany (schuette@zib.de).}
\and Tiejun Li\thanks{\textbf{Corresponding author.} LMAM and School of Mathematical Sciences, Center for Machine Learning Research and Center for Data Science, Peking University, Beijing 100871, P.R. China (tieli@pku.edu.cn).}
}

\pagestyle{myheadings}
\markboth{\uppercase{Riemannian Denoising Diffusion Probabilistic Models}}{\uppercase{Zichen Liu, Wei Zhang, Christof Schütte, and Tiejun Li}}
\maketitle

\begin{abstract}
We propose Riemannian Denoising Diffusion Probabilistic Models (RDDPMs) for learning distributions on submanifolds of Euclidean space that are level sets of functions, including most of the manifolds relevant to applications. Existing methods for generative modeling on manifolds rely on substantial geometric information such as geodesic curves or eigenfunctions of the Laplace-Beltrami operator and, as a result, they are limited to manifolds where such information is available. In contrast, our method, built on a projection scheme, can be applied to more general manifolds, as it only requires being able to evaluate the value and the first order derivatives of the function that defines the submanifold. We provide a theoretical analysis of our method in the continuous-time limit, which elucidates the connection between our RDDPMs and score-based generative models on manifolds. The capability of our method is demonstrated on distributions over complex manifolds implicitly represented as level sets, with applications in statistical mechanics and molecular dynamics. Our code is available at: \href{https://github.com/ZichenLiu1999/RiemannianDDPM}{https://github.com/ZichenLiu1999/RiemannianDDPM}.
\end{abstract}



\section{Introduction}\label{sec-intro}

Generative models have achieved remarkable success in learning data distributions across various fields. Among them, diffusion models stand out for their superior ability to generate high-quality samples that resemble the data distributions. Two prominent frameworks are Denoising Diffusion Probabilistic Models (DDPMs; \cite{NEURIPS2020_ddpm}), which minimize a variational bound in variational inference, and Score-based Generative Models (SGMs; \cite{song2019generative,song2021scorebased}), which learn the score function~\cite{hyvarinen2005estimation}. Both frameworks have demonstrated significant success in learning data distributions in Euclidean spaces.

In many scientific domains, data distributions are constrained to Riemannian manifolds rather than Euclidean spaces. For example, spheres are used in geographical sciences \cite{mathieu2020riemannian}, while $\mathrm{SE}(3)$ and $\mathrm{SO}(3)$ are considered in studying  protein structures \cite{watson2022broadly} and robotic movements \cite{simeonov2022neural}. Other manifolds include $\mathrm{SU}(3)$ in lattice quantum chromodynamics \cite{NEURIPS2023_ScalingRDM}, triangular meshes in 3D computer graphics \cite{hoppe1992surface}, and the Poincaré disk in cell development research \cite{klimovskaia2020poincare}. These applications highlight the need for developing generative models that can handle distributions on manifolds.

Several recent works have extended diffusion-based models to Riemannian manifolds. The Riemannian Score-based Generative Model (RSGM; \cite{de2022riemannian}) extends SGM to Riemannian manifolds by incorporating the heat kernel into the denoising score-matching objective. Since heat kernels on manifolds are generally intractable, RSGM approximates them using eigenfunction expansion or Varadhan’s approximation. Furthermore, RSGM leverages the exponential map to enable trajectory sampling on manifolds. The Riemannian Diffusion Model \cite{huang2022riemannian} adopts a variational diffusion model framework on Riemannian manifolds. It considers submanifolds embedded in an Euclidean space and utilizes a variational upper bound on the negative log-likelihood as loss function. Additionally, the Trivialized Diffusion Model \cite{zhu2025trivialized} is able to solve generative tasks on Lie groups by utilizing group property of the underlying manifold. Beyond diffusion-based models, flow-based generative models \cite{lou2020neural,mathieu2020riemannian,rozen2021moser,BenHamu2022MatchingNF,chen2024flow} extend continuous normalizing flows to manifolds. In particular, the method proposed in \cite{chen2024flow} accommodates manifolds with relatively general geometries, provided that their geodesic curves can be efficiently computed.

Despite these advancements, existing methods heavily rely on geometric information of manifolds, such as geodesics, exponential maps, heat kernels, or metrics. This restricts their applicability to manifolds where such information is readily available. In view of these limitations, we aim to develop algorithms that bypass these geometric dependencies, making them applicable to more general manifolds.

In this work, we introduce Riemannian Denoising Diffusion Probabilistic Models (RDDPMs), an extension of DDPMs to Riemannian submanifolds. A key ingredient is the projection scheme used in Monte Carlo methods for sampling under constraints, which allows us to develop Markov chains on manifolds with explicit transition densities. To our best knowledge, no successful extension of DDPMs to manifolds has been proposed prior to this study. The main advantages of our method over existing methods are summarized below.
\begin{itemize}
\item Our method is developed for submanifolds that are level sets of smooth functions in Euclidean space. This general setting includes most of the often studied manifolds such as spheres and matrix groups. More importantly, it fits well with applications where constraints are involved, e.g. applications in statistical mechanics and molecular dynamics.

\item Our method requires neither geodesic curves nor heat kernel, and it only relies on the computation of the value and the first-order derivatives of the function that defines the submanifold. This makes our approach applicable to more general manifolds.

\item We present a theoretical analysis for the loss function of our method in the continuous-time limit, elucidating its connection to the existing methods~\cite{de2022riemannian}. This analysis also shows the equivalence between loss functions derived from variational bound in variational inference and from learning score function. 
\end{itemize}

We successfully apply our method to distributions on manifolds studied in prior works, as well as new manifolds implicitly defined as level sets, such as the configuration space of alanine dipeptide with a fixed dihedral angle and the conserved Hamiltonian surface in phase space, both of which are challenging for existing methods due to their geometric complexity.

\section{Background}

\subsection{Riemannian submanifolds.}
We consider the zero level set $\mathcal{M}=\{x \in \mathbb{R}^n |\xi(x)=0\}$ of a smooth function $\xi:~\mathbb{R}^n\rightarrow\mathbb{R}^{n-d}$, where $1 \le d < n$.
We assume that $\mathcal{M}$ is non-empty and the matrix $\nabla\xi(x)\in \mathbb{R}^{n\times(n-d)}$, i.e.\ the Jacobian of $\xi$, has full rank at each $x\in \mathcal{M}$. 
Under this assumption, the regular value theorem~\cite[Corollary 5.9]{banyaga2004lectures} implies that $\mathcal{M}$ is a $d$-dimensional submanifold of $\mathbb{R}^n$. We further assume that $\mathcal{M}$ is a smooth compact connected manifold without boundary. 
The Riemannian metric on $\mathcal{M}$ is endowed from the standard Euclidean distance on $\mathbb{R}^n$.
For $x\in \mathcal{M}$, we denote by $T_x\mathcal{M}$ the tangent space of
$\mathcal{M}$ at~$x$. The orthogonal
projection matrix $P(x) \in \mathbb{R}^{n\times n}$ mapping
$T_x\mathbb{R}^n=\mathbb{R}^n$ to $T_x\mathcal{M}$ is given by $P(x) = I_n -
\nabla\xi(x) \big(\nabla\xi(x)^\top\nabla\xi(x)\big)^{-1} \nabla\xi(x)^\top$. Let $U_x\in\mathbb{R}^{n\times{d}}$ be a matrix whose column vectors form an
orthonormal basis of $T_{x}\mathcal{M}$ such that $U_x^\top U_x=I_d$. It is straightforward to verify that $P(x)=U_xU_x^\top$. 
The volume element over $\mathcal{M}$ is denoted by $\sigma_{\mathcal{M}}$.  All probability densities that appear in this paper refer to relative probability densities with respect to either $\sigma_{\mathcal{M}}$ or the product of $\sigma_{\mathcal{M}}$ over product spaces. 
For notational simplicity, we also use the shorthand $\int p(x^{(1:N)}) dx^{(1:N)} := \int_{\mathcal{M}}\dots \int_{\mathcal{M}} p(x^{(1)},\dots, x^{(N)}) d\sigma_{\mathcal{M}}(x^{(1)})\cdots d\sigma_{\mathcal{M}}(x^{(N)})$.

\subsection{Denoising diffusion probabilistic models.}
Denoising diffusion probabilistic models DDPMs \cite{pmlr-v37-sohl-dickstein15,NEURIPS2020_ddpm} employ a forward Markov chain to perturb data into noise and a reverse Markov chain to incrementally recover data from noise. The models are trained to minimize a variational bound on the negative log-likelihood. 
In the following, we formulate the general steps of DDPMs in the manifold setting. 

Assume that the data distribution is $q_0(x) d\sigma_{\mathcal{M}}(x)$. DDPMs are a class of generative models built on Markov chains. Specifically, states $x^{(1)}, \dots, x^{(N)}\in \mathcal{M}$ are generated by evolving the data $x^{(0)}$ according to a Markov chain on $\mathcal{M}$, which is called the forward process. The joint probability
density of $x^{(1)}, \dots, x^{(N)}$ given $x^{(0)}$ is 
\begin{equation}
  q(x^{(1:N)}\,|\,x^{(0)}) = \prod_{k=0}^{N-1} q(x^{(k+1)}\,|\,x^{(k)})\,,
    \label{joint-forward}
\end{equation}
where $q(x^{(k+1)}\,|\,x^{(k)})$ is the transition density of the forward process.
The generative process, also called the reverse process, is a Markov chain on $\mathcal{M}$ that is learnt to reproduce the data by reversing the forward process. Its joint probability density is 
\begin{equation}
  p_\theta(x^{(0:N)}) = p(x^{(N)}) \prod_{k=0}^{N-1} p_\theta(x^{(k)}\,|\,x^{(k+1)})\,, 
  \label{joint-backward}
\end{equation}
where $p(x^{(N)})$ is a (fixed) prior density, $p_\theta(x^{(k)}\,|\,x^{(k+1)})$ is the transition density of the reverse process, and $\theta$ is the parameter to be learnt. The probability density of
$x^{(0)}$ generated by the reverse process is therefore $p_\theta(x^{(0)})=
\int p_\theta(x^{(0:N)})\, dx^{(1:N)}$.
The learning objective is based on the standard variational bound on the negative log-likelihood. Specifically, 
using \eqref{joint-forward} and \eqref{joint-backward}, and applying Jensen's inequality, we can derive 
\begin{align}
\mathbb{E}_{q_0}\big(-\log p_\theta(x^{(0)})\big) 
= & \mathbb{E}_{q_0}\Big(-\log \int p_\theta(x^{(0:N)})\, dx^{(1:N)}\Big) \notag \\
=& \mathbb{E}_{q_0}\Big(-\log \int
\frac{p_\theta(x^{(0:N)})}{q(x^{(1:N)}\,|\,x^{(0)})}
q(x^{(1:N)}\,|\,x^{(0)})\, dx^{(1:N)}\Big) \notag \\
\le & \mathbb{E}_{\mathbb{Q}^{(N)}}\Big( -\log \frac{p_\theta(x^{(0:N)})}{q(x^{(1:N)}\,|\,x^{(0)})}\Big)\notag \\
=& \mathbb{E}_{\mathbb{Q}^{(N)}}\Big( -\log p(x^{(N)}) - \sum_{k=0}^{N-1}
\log\frac{p_\theta(x^{(k)}\,|\,x^{(k+1)})}{q(x^{(k+1)}\,|\,x^{(k)})}\Big)\,,  \label{variational-bound}
\end{align}
where $\mathbb{E}_{q_0}, \mathbb{E}_{\mathbb{Q}^{(N)}}$ denote the expectation with respect to the data distribution on $\mathcal{M}$, and the expectation with respect to the joint density $q(x^{(0:N)})$ under the forward process, respectively. In order to derive an explicit training objective, we have to construct Markov chains on $\mathcal{M}$ with explicit transition densities. We discuss how this can be achieved in the next section.

We conclude this section by reformulating the variational bound (\ref{variational-bound}) 
using relative entropy (see~\cite{NEURIPS2021_likelihood_training_sbdm} for a similar formulation of score-based diffusion models). Recall that the relative entropy, or Kullback-Leibler (KL) divergence, from a probability density $Q_2$ to another probability density $Q_1$ in the same measure space, where $Q_1$ is absolutely continuous with respect to $Q_2$, is defined as
$H(Q_1\,|\,Q_2) := \mathbb{E}_{Q_1}\Big(\log \frac{Q_1}{Q_2}\Big)$.
For simplicity, we also use the same notation for two probability measures.
Adding the term $\mathbb{E}_{q_0}(\log q_0)$ to both sides of the inequality (\ref{variational-bound}), we see that it is equivalent to the so-called data processing inequality
\begin{equation}
  H(q_0\,|\,p_\theta) \le
  H(\overleftarrow{\mathbb{Q}}^{(N)}\,|\,\mathbb{P}^{(N)}_\theta) \,,
  \label{relative-entropy-formulation-of-variational-bound}
\end{equation}
where the upper bound is the relative entropy from the path measure
 $\mathbb{P}^{(N)}_\theta$ of the reverse process to the path measure $\overleftarrow{\mathbb{Q}}^{(N)}$
of the forward process (the arrow in the notation indicates that paths of the forward process are viewed backwardly).
Therefore, learning DDPMs using the variational bound (\ref{variational-bound}) can be viewed as approximating probability measures in path space by the cross-entropy method~\cite{zwhws2014}.

\section{Method}\label{sec-method}

\subsection{Projection scheme.}\label{subsec-projection-scheme}
We recall a projection scheme from Monte Carlo sampling methods on manifolds~\cite{projection_diffusion,goodman-submanifold,hmc-submanifold-tony,LelievreStoltzZhang2022}, and we show that it allows us to construct Markov chains on $\mathcal{M}$ with tractable transition densities. 

Given $x\in\mathcal{M}$ and a tangent vector $v\in T_x\mathcal{M}$ that is drawn from the standard Gaussian distribution on $T_x\mathcal{M}$, we compute the intermediate state $x'=x+ \sigma^2 b(x)+\sigma v\in \mathbb{R}^n$, where $\sigma>0$ is a positive constant and $b: \mathbb{R}^n \rightarrow \mathbb{R}^n$ is a smooth function. In general, $x'$ does not belong to $\mathcal{M}$. We consider the projection $y\in \mathcal{M}$ of $x'$ onto $\mathcal{M}$ along an orthogonal direction in the space spanned by column vectors of $\nabla\xi(x)$. Precisely, the projected state $y$ is found by (numerically) solving the constraint equation for $c\in \mathbb{R}^{n-d}$ 
\begin{equation}
  y = x + \sigma^2 b(x) + \sigma v + \nabla\xi(x)c,  ~~ \mbox{s.t.}~~ \xi(y) =  0\in \mathbb{R}^{n-d} \,.
  \label{projection}
\end{equation}
This projection scheme can be seen as a discretization of a certain stochastic process on the manifold $\mathcal{M}$, where $\sigma^2$ corresponds to the step size~\cite{projection_diffusion}. The choice of $b$ will affect the final invariant distribution and the convergence rate to equilibrium of the resulting Markov chain (see Section~\ref{subsec-algo-details} for further discussion). There are $n-d$ constraints in \eqref{projection} with the same number of unknown variables. In particular, when $\xi$ is scalar-valued, i.e., $n-d=1$, solving \eqref{projection} amounts to finding a root of a (nonlinear) scalar function.

Algorithm~\ref{Newton_solver} shows the details of solving the linear equations. Let $k_{\mathrm{iter}}$ denote the number of Newton iterations, and $C_{\xi}$ be the computational cost of evaluating $\nabla \xi$. The complexity of solving the linear equations in line 4 of Algorithm~\ref{Newton_solver} is $\mathcal{O}((n-d)^3)$. Thus, the total complexity of Newton’s method is $\mathcal{O}(k_{\mathrm{iter}}(C_{\xi} +(n-d)^3))$. When the step size $\sigma$ is sufficiently small, Newton's method typically converges in a few iterations. Therefore, the computational cost of Newton's method is primarily determined by the cost of computing $\nabla \xi$ (i.e., $C_{\xi}$) and the codimension $n-d$ of the manifold.

While for some vectors $v$ multiple solutions to \eqref{projection} may exist in theory, the projected state, and hence the resulting Markov chain built on a numerical solver, is uniquely defined as long as the numerical solver finds one solution in a deterministic way. For example, this is the case when Newton's method is adopted to solve \eqref{projection} with fixed initial condition $c=0$. When no solution can be found for some $v$, we can either resample the tangent vector $v$ or resample a new path (see Section~\ref{subsec-algo-details} for further discussion). Let $\mathcal{F}_{x}^{(\sigma)}$ be the set of $v$ for which a solution can be found and denote by $\epsilon_{x}^{(\sigma)}=\mathbb{P}(v\notin \mathcal{F}_{x}^{(\sigma)})$, i.e.\ the probability that no solution can be found. We denote $\mathcal{M}_{x}^{(\sigma)}$ the set of all states in $\mathcal{M}$ that can be reached from $x$ by solving \eqref{projection} with certain~$v\in \mathcal{F}_{x}^{(\sigma)}$. Notice that $\epsilon_{x}^{(\sigma)}=0$ when $\sigma=0$, because in this case $c=0$ is a solution to \eqref{projection} for any $v$. Therefore, it is expected that $\epsilon_{x}^{(\sigma)}\rightarrow 0$ as $\sigma$ decreases to zero. That is, when $\sigma$ is small, a solution to \eqref{projection} can be found with a probability that is close to one.

To derive the transition density of jumping to $y$ from $x$, we notice that, by applying the orthogonal projection matrix $P(x)$ to both sides of \eqref{projection} and using the fact that $P(x)\nabla\xi(x)=0$, we have the relation $\sigma v=P(x)(y-x-\sigma^2 b(x))$. This indicates that, given a state $x\in\mathcal{M}$ and $y\in \mathcal{M}_{x}^{(\sigma)}$, there is a unique tangent vector $v\in \mathcal{F}_{x}^{(\sigma)}\subseteq T_x\mathcal{M}$ that leads to $y$ by solving \eqref{projection}. In other words, the mapping from $v\in \mathcal{F}_{x}^{(\sigma)}$ to $y\in \mathcal{M}_{x}^{(\sigma)}$ is a bijection. Moreover, its inverse is explicitly given by $G_{x}^{(\sigma)}: \mathcal{M}_{x}^{(\sigma)}\rightarrow \mathcal{F}_{x}^{(\sigma)} \subseteq T_{x}\mathcal{M}$, where
\begin{equation}\label{general_g-x}
  G_{x}^{(\sigma)}(y;b) = \frac{1}{\sigma} P(x)(y-x - \sigma^2 b(x))\,.
\end{equation}
To simplify the notation, we also write $G_{x}^{(\sigma)}(y)$ when omitting the dependence on $b$ does not cause ambiguity. Recall that $U_x, U_y\in \mathbb{R}^{n\times d}$ denote the matrices whose columns form an orthonormal basis of $T_x\mathcal{M}$ and $T_y\mathcal{M}$, respectively. Using \eqref{general_g-x}, we can derive 
\begin{equation}
\det (D G_{x}^{(\sigma)}(y)) = \sigma^{-d} \det (U_x^\top U_y)\,,
\end{equation}
where the left-hand side denotes the determinant of the Jacobian $D G_{x}^{(\sigma)}(y): T_{y}\mathcal{M}\rightarrow T_vT_{x}\mathcal{M}\cong \mathbb{R}^d$ of the map $G_{x}^{(\sigma)}$ at $y$. See the texts above (4.5) in Section 4 of \cite{LelievreStoltzZhang2022} for detailed discussions. Since $v$ is a Gaussian variable confined in $\mathcal{F}_{x}^{(\sigma)}$ (with a normalizing constant rescaled by $(1-\epsilon_{x}^{(\sigma)})^{-1}$), applying the change of variables formula for probability densities, we obtain the probability density of $y$ conditioned on $x$:
 \begin{align}
   q(y\,|\,x) =& (2\pi)^{-\frac{d}{2}} (1-\epsilon^{(\sigma)}_{x})^{-1}\mathrm{e}^{-\frac{1}{2} |G_{x}^{(\sigma)}(y)|^2} |\det D G_{x}^{(\sigma)}(y)| \notag \\
   =& (2\pi\sigma^2)^{-\frac{d}{2}} (1-\epsilon^{(\sigma)}_{x})^{-1} \mathrm{e}^{-\frac{1}{2} |G_{x}^{(\sigma)}(y)|^2}
  |\det (U_x^\top U_y)|\,, \quad y\in \mathcal{M}_{x}^{(\sigma)}\,. \label{transition-density-single-step}
\end{align}
For $y \in \mathcal{M} \setminus \mathcal{M}_{x}^{(\sigma)}$, the probability density is zero, i.e., $q(y\,|\,x) = 0$.

\subsection{Forward process.}\label{subsec-forward-process}
We construct the forward process in our model as a Markov chain on $\mathcal{M}$ whose transitions are defined by the projection scheme in \eqref{projection}.
Specifically, given the current state $x^{(k)}\in \mathcal{M}$ at step $k$, where $k=0,1,\dots, N-1$, the next state $x^{(k+1)}\in \mathcal{M}$ is determined by solving the constraint equation (for $c\in \mathbb{R}^{n-d}$):
\begin{equation} \label{projection-kstep-forward}
x^{(k+1)} = x^{(k)} + \sigma_k^2 b(x^{(k)}) + \sigma_k v^{(k)} + \nabla\xi(x^{(k)})c, ~~\mbox{s.t.}~~ \xi(x^{(k+1)}) =  0\in\mathbb{R}^{n-d}\,,
\end{equation}
where $\sigma_k>0$ and $v^{(k)}\in T_{x^{(k)}}\mathcal{M}$ is a standard Gaussian variable in $T_{x^{(k)}}\mathcal{M}$. According to \eqref{joint-forward} and \eqref{transition-density-single-step}, we obtain the transition probability density of the forward process as 
\begin{equation}\label{trans-forward-projection}
q(x^{(k+1)}\,|\,x^{(k)}) = (2\pi\sigma_k^2)^{-\frac{d}{2}}|\det (U_{x^{(k)}}^\top U_{x^{(k+1)}})| \big(1-\epsilon_{x^{(k)}}^{(\sigma_k)}\big)^{-1} \exp\left(-\frac{1}{2}\left|G_{x^{(k)}}^{(\sigma_k)}(x^{(k+1)};b)\right|^2\right)\,,
\end{equation}
where
the function $G_{x^{(k)}}^{(\sigma_k)}(x^{(k+1)};b)$ is defined  in \eqref{general_g-x}.

\subsection{Reverse process.}
The reverse process in our model is a Markov chain on $\mathcal{M}$
whose transitions (from $x^{(k+1)}$ to $x^{(k)}$) are defined by the constraint equation
\begin{align}
& x^{(k)} = x^{(k+1)} - \beta_{k+1}^2 b(x^{(k+1)}) + \beta_{k+1}^2 s^{(k+1),\theta}(x^{(k+1)}) + \beta_{k+1} \bar{v}^{(k+1)} + \nabla\xi(x^{(k+1)})c, \notag \\
\mbox{s.t.}~~& \xi(x^{(k)}) = 0 \,, \label{projection-kstep-backward}
\end{align}
for $k=N-1, N-2, \dots, 0$, where $\beta_{k+1}>0$, $\bar{v}^{(k+1)}$ is a standard Gaussian variable in $T_{x^{(k+1)}}\mathcal{M}$, and $s^{(k+1),\theta}(x^{(k+1)})\in\mathbb{R}^n$ depends on the learning parameter $\theta$. Combining \eqref{joint-backward} and \eqref{transition-density-single-step}, we obtain the transition density of the reverse process as
\begin{align}
    &p_\theta(x^{(k)}\,|\,x^{(k+1)}) \notag \\
    =& (2\pi\beta_{k+1}^2)^{-\frac{d}{2}} (1-\epsilon_{x^{(k+1)},\theta}^{(\beta_{k+1})} )^{-1} |\det (U_{x^{(k+1)}}^\top U_{x^{(k)}})|  \exp\left(-\frac{1}{2}\left|G_{x^{(k+1)}}^{(\beta_{k+1})}(x^{(k)}; s^{(k+1),\theta}-b)\right|^2\right)\,. \label{trans-backward-projection}
\end{align}
In the above, $\epsilon_{x^{(k+1)},\theta}^{(\beta_{k+1})}$ denotes the probability of having $\bar{v}^{(k+1)}$ with which no solution to (\ref{projection-kstep-backward}) can be found and $G_{x^{(k+1)}}^{(\beta_{k+1})}(x^{(k)}; s^{(k+1),\theta}-b)$ is defined in \eqref{general_g-x}.

\subsection{Training objective.}
\label{subsec-training-objective}
The training objective follows directly from the variational bound~(\ref{variational-bound}) on
the negative log-likelihood, as well as the explicit expressions of transition densities. Concretely, substituting \eqref{trans-forward-projection} and \eqref{trans-backward-projection} into the last line of (\ref{variational-bound}), we get
\begin{equation}
  \mathbb{E}_{q_0}[-\log p_\theta(x^{(0)})]   \le  \mbox{Loss}^{(N)}(\theta) + C^{(N)}  \,,
  \label{variational-bound-involving-loss}
\end{equation}
where 

\begin{equation}
\mbox{Loss}^{(N)}(\theta)=\frac{1}{2}\mathbb{E}_{\mathbb{Q}^{(N)}}\sum_{k=0}^{N-1} \left|G_{x^{(k+1)}}^{(\beta_{k+1})}(x^{(k)}; s^{(k+1),\theta}-b)\right|^2
\label{training-loss}
\end{equation}
is our objective for training the parameter $\theta$ in the reverse process  
(recall that $\mathbb{E}_{\mathbb{Q}^{(N)}}$ denotes the expectation with respect to the forward process), the constant
\begin{align}
C^{(N)} = &-\mathbb{E}_{\mathbb{Q}^{(N)}}\bigg[\frac{1}{2}\sum_{k=0}^{N-1}\left|G_{x^{(k)}}^{(\sigma_k)}(x^{(k+1)};b)\right|^2+\log p(x^{(N)})\bigg] \notag \\
& +d \sum_{k=0}^{N-1} \log \frac{\beta_{k+1}}{\sigma_k} - \mathbb{E}_{\mathbb{Q}^{(N)}} \sum_{k=0}^{N-1}\log \big(1-\epsilon_{x^{(k)}}^{(\sigma_k)}\big) \label{constant-c}
\end{align}
is independent of $\theta$, and we have used the inequality $\log\left(1-\epsilon_{x^{(k+1)},\theta}^{(\beta_{k+1})}\right)\le~0$ in deriving \eqref{variational-bound-involving-loss}.

\subsection{Algorithmic details.}\label{subsec-algo-details}
The algorithms for sampling the trajectories of the forward and reverse processes are summarized in Algorithms~\ref{algo_forward} and~\ref{algo_generation}, respectively, both of which involve solving constraint equations, as detailed in Algorithm~\ref{Newton_solver}. The training algorithm is summarized in Algorithm~\ref{algo_training}. In the following, we discuss several algorithmic details of our method.

\subsubsection{Choice of the Markov chain.}
The total number of steps $N$ should be large enough so that the forward Markov chain can approximately reach equilibrium starting from the data distribution. To simplify algorithm implementation, we may set $\beta_{k+1}=\sigma_k$, a choice that is also supported by our theoretical derivation in the continuous-time limit (see Theorem \ref{thm-continuous-limit}). While larger $\sigma_k$, $\beta_{k+1}$ allow the Markov chains to make larger jumps, their sizes should be chosen properly (depending on the manifold) so that the solution to the constraint equations \eqref{projection-kstep-forward} and \eqref{projection-kstep-backward} can be found with high probability.

For relatively simple manifolds, we can simply choose $b=0$ in the forward Markov chain~\eqref{projection-kstep-forward}. As we will discuss in Section~\ref{sec-continuous-limit}, in the continuous-time limit, the Markov chain with $b=0$ converges to the Brownian motion on $\mathcal{M}$ up to a rescaling of time (\eqref{sde-manifold} with $b=0$), whose invariant distribution is the uniform distribution on $\mathcal{M}$. Motivated by this fact, we can choose the prior (see line 2 in Algorithm~\ref{algo_generation}) as the uniform distribution on $\mathcal{M}$. When $\mathcal{M}$ is non-compact or when the convergence of Markov chain to equilibrium is slow with $b=0$, we can choose non-zero $b$ such as $b=-\nabla V$, i.e. the (full space) gradient of some function $V:\mathbb{R}^n \rightarrow \mathbb{R}$ in the ambient space. In this case, we choose the prior as the invariant distribution of the forward process, and sampling the prior can be done by simulating a single long trajectory of the forward process.

\subsubsection{Method for solving constraint equations.}
As in Monte Carlo sampling methods on submanifolds, we employ Newton's method to solve the constraint equations~\eqref{projection-kstep-forward} and \eqref{projection-kstep-backward}. This method has (local) quadratic convergence and its implementation is simple. In most cases, a solution with high precision can be found within a few iteration steps (usually less than 5 steps). When no solution is found, one can re-generate the state by sampling a new tangent vector or re-generate the entire trajectory. Our implementation of Newton's method is summarized in Algorithm~\ref{Newton_solver}.

\subsubsection{Generation of trajectory data.}
The optimal parameter $\theta$ is sought by minimizing the objective \eqref{training-loss}, which requires the trajectory data of the forward process. Although sampling trajectories to evaluate the loss function may seem computationally expensive, the computational cost can be alleviated by using a pre-prepared trajectory dataset that is updated during training with a tunable frequency. In our implementation, trajectories are initially sampled in a preparatory step, and the model is trained with mini-batches from this trajectory dataset, which is periodically updated (see line 2 and lines 10--12 in Algorithm~\ref{algo_training}).

\begin{algorithm}[ht]
\caption{Sampling trajectory of forward process}\label{algo_forward}
\begin{algorithmic}[1]
\STATE \textbf{Input: } $x^{(0)}\in\mathcal{M}$, constants $\sigma_{k}$, function $b$, and integer $N$
\FOR {$k = 0$ to $N-1$}
    \STATE sample $z^{(k)}\sim \mathcal{N}(0,I_n)$ and set $v^{(k)}= P(x^{(k)})z^{(k)}$
    \STATE set $x^{(k+\frac{1}{2})}:= x^{(k)} + \sigma_k^2 b(x^{(k)})+ \sigma_k v^{(k)}$
    \STATE $c, \mathrm{flag} = \text{newton\_solver}(x^{(k)},x^{(k+\frac{1}{2})}; \xi)$. \hfill \MYCOMMENT{solve~(\ref{projection-kstep-forward}) by Algorithm~\ref{Newton_solver}}
    \IF {flag == true} 
    \STATE set $x^{(k+1)} := x^{(k+\frac{1}{2})} + \nabla\xi(x^{(k)})c$
    \ELSE
    \STATE discard the trajectory and re-generate \hfill
    \MYCOMMENT{alternatively, go back to line 3}
    \ENDIF
\ENDFOR
\STATE \textbf{Return} $(x^{(0)},x^{(1)},\dots,x^{(N)})$
\end{algorithmic}
\end{algorithm}

\begin{algorithm}[ht]
\caption{Sampling trajectory of reverse process}\label{algo_generation}
\begin{algorithmic}[1]
\STATE \textbf{Input: } trained functions $(s^{(k+1),\theta}(x))_{0\le k\le N-1}$, constants $\beta_{k}$, function $b$, and integer $N$
\STATE draw sample $x^{(N)}$ from the prior distribution $p(x^{(N)})$
\FOR {$k = N-1$ to $0$}
    \STATE sample $\bar{z}^{(k+1)}\sim \mathcal{N}(0,I_n)$ and set $\bar{v}^{(k+1)}= P(x^{(k+1)})\bar{z}^{(k+1)}$
    \STATE set $x^{(k+\frac{1}{2})} = x^{(k+1)} +\beta_{k+1}^2  P(x^{(k+1)}) \left(s^{(k+1),\theta}-b\right)(x^{(k+1)})   +\beta_{k+1} \bar{v}^{(k+1)}$
    \STATE $c, \mathrm{flag}=\text{newton\_solver}(x^{(k+1)}, x^{(k+\frac{1}{2})};\xi)$ \hfill\MYCOMMENT{solve (\ref{projection-kstep-backward}) by Algorithm~\ref{Newton_solver}}
    \IF {flag == true}
    \STATE $x^{(k)} := x^{(k+\frac{1}{2})} + \nabla\xi(x^{(k+1)})c$
    \ELSE
    \STATE discard the trajectory and re-generate \hfill \MYCOMMENT{alternatively, go back to line 4}
    \ENDIF
\ENDFOR
\STATE \textbf{Return} $(x^{(N)},x^{(N-1)},\dots,x^{(0)})$
\end{algorithmic}
\end{algorithm}

\begin{algorithm}[ht]
\caption{newton\_solver($x,x';\xi$) \hfill{\textcolor{gray}{// solve $\xi(x'+\nabla\xi(x)c)=0$ by Newton's method}}}\label{Newton_solver}
\begin{algorithmic}[1]
\STATE \textbf{Input: } $x\in \mathcal{M}$, $x'\in\mathbb{R}^n$, $\xi: \mathbb{R}^n\rightarrow \mathbb{R}^{n-d}$, maximal steps $n_{\mathrm{step}}$, tolerance $\mathrm{tol}>0$
\STATE \textbf{Initialization}: set $c=0\in\mathbb{R}^{n-d}$
\FOR {$k = 0$ to $n_{\mathrm{step}}-1$}
    \STATE Solve linear equation $\big[\nabla \xi\big(x'+\nabla\xi(x)c\big)^\top\nabla\xi(x)\big] u= -\xi\big(x'+\nabla \xi (x)c\big)$ for $u\in\mathbb{R}^{n-d}$
    \STATE $c \leftarrow c+u$
    \IF {$|\xi(x'+\nabla\xi(x) c)|<\mathrm{tol}$}
    \STATE \textbf{Return} c, true
    \ENDIF
\ENDFOR
\STATE \textbf{Return} c, false
\end{algorithmic}
\end{algorithm}

\begin{algorithm}[ht]
\caption{Training procedure}\label{algo_training}
\begin{algorithmic}[1]
\STATE \textbf{Input: } training data $(y^{i})_{1\le i\le M}$, functions $(s^{(k+1),\theta}(x))_{0\le k\le N-1}$, constants $\sigma_{k}, \beta_{k}>0$, function $b: \mathbb{R}^n\rightarrow\mathbb{R}^n$, integer $N$, batch size $B>0$, number of total training epochs $N_{\mathrm{epoch}}$, integer $l_{\mathrm{f}}>0$, learning rate $r>0$.
\STATE generate a path $(x^{(0),i},x^{(1),i},\dots,x^{(N),i})$ using Algorithm \ref{algo_forward}, for each $x^{(0),i}=y^{i}$
\FOR{$l = 1$ to $N_{\mathrm{epoch}}$}
\FOR{$j = 1$ to $\lfloor M/B \rfloor $}
\STATE sample a min-batch $\mathcal{I}=(i_1,i_2,\dots,i_{B})$ from the set of indices $\{1,2,\dots, M\}$
\STATE calculate loss: $\ell(\theta)=\frac{1}{2|\mathcal{I}|}\sum\limits_{i\in\mathcal{I}}\sum\limits_{k=0}^{N-1} \left|G_{x^{(k+1),i}}^{(\beta_{k+1})}(x^{(k),i}; s^{(k+1),\theta}-b)\right|^2$
\STATE $\theta=\operatorname{optimizer\_update}(\theta, \ell(\theta), r)$ 
\ENDFOR
\STATE \MYCOMMENT{update trajectories every $l_{\mathrm{f}}$ epochs}
\IF{$l\, \%\, l_{\mathrm{f}} == 0$} 
\STATE re-generate paths $(x^{(0),i},x^{(1),i},\dots,x^{(N),i})$ using Algorithm \ref{algo_forward}, for each $x^{(0),i}=y^{i}$
\ENDIF
\ENDFOR
\STATE \textbf{Return} $\theta$
\end{algorithmic}
\end{algorithm}

\section{Theoretical results}
\label{sec-continuous-limit}
In this section, we study the continuous-time limit of our proposed method. Let $T>0$ and $g: [0,T]\rightarrow \mathbb{R}^+$ be a continuous function.  Define $h=\frac{T}{N}$ and consider the case where $\sigma_k = \sqrt{h} g(kh)$, for $k=0,1,\dots,N-1$. It is shown in~\cite{projection_diffusion} that the forward process (\ref{projection-kstep-forward}) converges strongly to the SDE on $\mathcal{M}$
\begin{equation}\label{sde-manifold}
  dX_t =  g^2(t) P(X_t) b(X_t) dt + g(t) dW^\mathcal{M}_t\,,\quad ~t\in[0,T]\,,
\end{equation}
where $W^\mathcal{M}_t$ is a Brownian motion over $\mathcal{M}$. Denote by $p(\cdot,t)$ the probability density of $X_t$ with respect to $\sigma_{\mathcal{M}}$ at time $t\in [0,T]$. We have the following result, which characterizes the loss function in \eqref{training-loss} as $N\rightarrow +\infty$.

\begin{theorem}\label{thm-continuous-limit}
Let $T>0$ and $g: [0,T]\rightarrow \mathbb{R}^+$ be a continuous function. Define $h=\frac{T}{N}$ and $t_k=kh$, for $k=0,1,\dots,N-1$. Assume that $\sigma_k = \beta_{k+1} = \sqrt{h} g(t_k)$. Also assume that there is a $C^1$ function $s_\theta: \mathbb{R}^n\times [0,T]\rightarrow \mathbb{R}^n$ such that $s^{(k+1),\theta}(x) = s_\theta(x,t_{k+1})$ for all $k=0,1,\dots,N-1$ and $x\in\mathcal{M}$. For the loss function defined in \eqref{training-loss}, we have 
\begin{align*}
& \lim_{N\rightarrow +\infty}\left(\mathrm{Loss}^{(N)}(\theta) -\frac{1}{2}\mathbb{E}_{\mathbb{Q}^{(N)}}\sum_{k=0}^{N-1}\big|v^{(k)}\big|^2 \right) \\
= & \mathbb{E}_{\mathbb{Q}}\Big[\frac{1}{2}\int_{0}^T \big| P(X_t)s_{\theta}(X_t, t)-\nabla_{\mathcal{M}} \log p(X_t,t)\big|^2   g^2(t)\,dt \\
& + \int_{0}^T \Big(P(X_t)b(X_t)-\frac{1}{2}\nabla_{\mathcal{M}} \log  p(X_t,t)\Big)  \cdot \nabla_{\mathcal{M}} \log p(X_t,t)\, g^2(t)\, dt \Big]\, ,
\end{align*}
where $\mathbb{E}_{\mathbb{Q}}$ on the right-hand side denotes the expectation with respect to the paths of SDE~(\ref{sde-manifold}) and $\nabla_{\mathcal{M}}$ denotes the gradient operator on $\mathcal{M}$.
\end{theorem}

Based on Theorem~\ref{thm-continuous-limit}, the variational bound~(\ref{variational-bound}), and its relative entropy formulation in (\ref{relative-entropy-formulation-of-variational-bound}), we obtain the following corollary, which states that our RDDPMs learn the score function as $N\rightarrow +\infty$. It elucidates the connection between our RDDPMs and Riemannian score-based generative models \cite{de2022riemannian} as $N\rightarrow +\infty$.  

\begin{corollary}
Under the same assumptions of Theorem~\ref{thm-continuous-limit}, we have, for any parameter $\theta$,
\begin{equation}
\lim_{N\rightarrow +\infty} H\big(\overleftarrow{\mathbb{Q}}^{(N)}\,|\,\mathbb{P}^{(N)}_\theta\big) =  \frac{1}{2}\mathbb{E}_{\mathbb{Q}}\int_{0}^T \big| P(X_t)s_{\theta}(X_t, t)-\nabla_{\mathcal{M}} \log p(X_t,t)\big|^2 g^2(t) dt = H(\overleftarrow{\mathbb{Q}}\,|\,\mathbb{P}_\theta),
\end{equation}
where $\overleftarrow{\mathbb{Q}}$ denotes the path measure of the time-reversal of SDE~(\ref{sde-manifold}), and $\mathbb{P}_\theta$ denotes the path measure of the SDE 
\begin{equation}\label{sde-manifold-reversed-theta}
dY_t =  g^2(T-t) P(Y_t) \big(-b(Y_t) + s_\theta(Y_t, T-t)\big) dt 
+ g(T-t) dW^\mathcal{M}_t\,, ~~ t\in[0,T]\,,
\end{equation}
starting from $Y_0=X_T$.
  \label{corollary-continuous-limit-kl}
\end{corollary}

The proofs of Theorem~\ref{thm-continuous-limit} and
Corollary~\ref{corollary-continuous-limit-kl} are presented in Appendix \ref{app-sec-proofs}.

\section{Experiments}

We evaluate our method on distributions defined on mesh manifolds, the high-dimensional special orthogonal group (both with and without a nonlinear transformation), conserved Hamiltonian surfaces in phase space, and molecular conformations under constraints. The last three novel datasets have not been studied by existing methods. Further experimental details and results can be found in Appendix~\ref{Exp_details}. In particular, parameters for each example are summarized in Table~\ref{Hyperparameters} in Appendix~\ref{Exp_details}. 

\subsection{Mesh data on learned manifolds.}
Our method can effectively handle manifolds with general geometries. For demonstration, we examine the Stanford Bunny \cite{turk1994zippered} and Spot the Cow \cite{crane2013robust}, two manifolds defined by triangle meshes. To create the target distribution, we follow the approach in \cite{jo2024generativemodelingmanifoldsmixture} and \cite{chen2024flow}, which utilizes the $k$-th clamped eigenfunction of the Laplace-Beltrami operator on meshes.

Similar to \cite{rozen2021moser}, we first learn a function $\xi: \mathbb{R}^3\rightarrow \mathbb{R}$, whose zero level set matches the manifold. We adopt the approach in \cite{gropp2020implicit}, where $\xi$ is represented by a neural network and is trained such that on mesh data $\xi$ is close to zero and $|\nabla\xi|$ is close to one. Using this approach, we obtain a function $\xi$ whose order of magnitude is $10^{-2}$ on mesh data. Then, we perform a further refinement to the dataset such that all points belong to the learned manifold $\mathcal{M}=\{x\in \mathbb{R}^3|\xi(x)=0\}$ up to a small error $10^{-5}$. The maximal distance between the original data and the refined data is smaller than $0.017$.

The function $b$ in the forward Markov chain \eqref{projection-kstep-forward} is set to be the zero function, and the prior distribution $p(x^{(N)})$ is a uniform distribution on the meshes (see discussions in Section~\ref{subsec-algo-details}). We perform the training with the learned function $\xi$. 

In Table \ref{sdf_compare}, we present the negative log-likelihood (NLL) on the test set, estimated via the second line of (\ref{variational-bound}). Our method outperforms existing manifold-based methods, including RFM \cite{chen2024flow} and LogBM \cite{jo2024generativemodelingmanifoldsmixture}. One possible explanation for this improvement is that existing methods require computing the premetric on meshes through infinite series (see Equation (16) in \cite{chen2024flow}), which introduces bias due to truncation in practice. In principle, our method is unbiased, as it does not require computing distances on the manifold. Figure~\ref{fig:mesh_result} visualizes the generated samples, demonstrating good agreement with the target data distribution.

\begin{table}[ht]
\centering
\begin{tabular}{lcccc}
\toprule
 & \multicolumn{2}{c}{\textbf{Stanford Bunny}} & \multicolumn{2}{c}{\textbf{Spot the Cow}} \\
\cmidrule(lr){2-3} \cmidrule(lr){4-5}
 & $ k = 50 $ & $ k = 100 $ & $ k = 50 $ & $ k = 100 $ \\
\midrule
RFM w/ Diff. & 1.48\scriptsize{$\pm$0.01} & 1.53\scriptsize{$\pm$0.01} & 0.95\scriptsize{$\pm$0.05} & 1.08\scriptsize{$\pm$0.05} \\
RFM w/ Bihar. & 1.55\scriptsize{$\pm$0.01} & 1.49\scriptsize{$\pm$0.01} & 1.08\scriptsize{$\pm$0.05} & 1.29\scriptsize{$\pm$0.05} \\
LogBM w/ Diff. & 1.42\scriptsize{$\pm$0.01} & 1.41\scriptsize{$\pm$0.00} &  0.99\scriptsize{$\pm$0.03} & 0.97\scriptsize{$\pm$0.03}\\
LogBM w/ Bihar. & 1.55\scriptsize{$\pm$0.02} & 1.45\scriptsize{$\pm$0.01} & 1.09\scriptsize{$\pm$0.06} & 0.97\scriptsize{$\pm$0.02}  \\
\midrule
\hspace{-0.2cm} \textit{\scriptsize Ours}\\
RDDPM  & \textbf{1.36}\scriptsize{$\pm$0.00}  & \textbf{1.31}\scriptsize{$\pm$0.01} &  \textbf{0.84}\scriptsize{$\pm$0.00} & \textbf{0.77}\scriptsize{$\pm$0.00} \\
\bottomrule
\end{tabular}
\caption{Test negative log-likelihood (NLL) on mesh datasets. Lower values indicate better performance. The table shows the mean and the standard deviation of the NLL over five independent runs. For RFM and LogBM, \textit{with Diff} and \textit{with Bihar} refer to different weighting functions used in computing the spectral distances (diffusion distance and biharmonic distance, respectively).}\label{sdf_compare}
\end{table}

\begin{figure}[ht]
\centering
\includegraphics[width=0.9\linewidth]{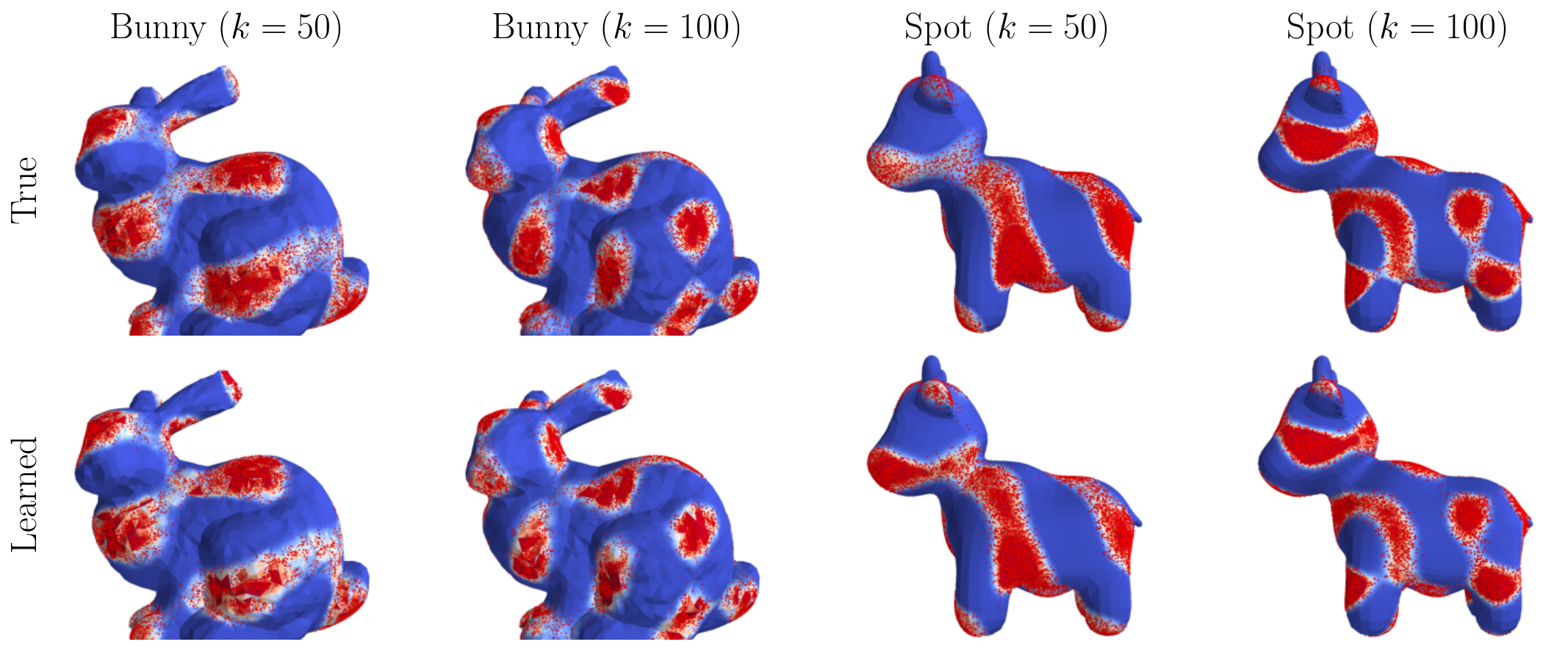}
\caption{First row: datasets and true distributions. Second row: generated samples and distributions from the trained models.}\label{fig:mesh_result}
\end{figure}

\subsection{High-dimensional special orthogonal group.}
We apply our method to the special orthogonal group $\mathrm{SO}(10)$, viewed as a $45$-dimensional submanifold embedded in $\mathbb{R}^{100}$. This manifold can be characterized as (one of the two connected components of) the zero level set of the map $\xi: \mathbb{R}^{100} \rightarrow \mathbb{R}^{55}$, whose components are defined by the upper triangular entries of the matrix $S^\top S - I_{10}$, where $S$ is a $10 \times 10$ matrix and $I_{10}$ denotes the identity matrix of size $10$.

The dataset is sampled from a multimodal distribution on $\mathrm{SO}(10)$  with $5$ modes. As in the previous example, we choose $b$ in the forward Markov chain to be zero and the prior distribution to be the uniform distribution on the manifold.

To assess the quality of generated data, we consider the statistics $\operatorname{tr}(S)$, $\operatorname{tr}(S^2)$, $\operatorname{tr}(S^4)$, and $\operatorname{tr}(S^5)$, where $\operatorname{tr}$ denotes the trace operator of matrices. Figure~\ref{fig:SOn_result_5} indicates that our learned model can generate the data distribution accurately. What is more, the distributions of the forward process at intermediate steps are also faithfully reproduced.

\begin{figure}[ht]
\centering
\includegraphics[width=0.95\linewidth]{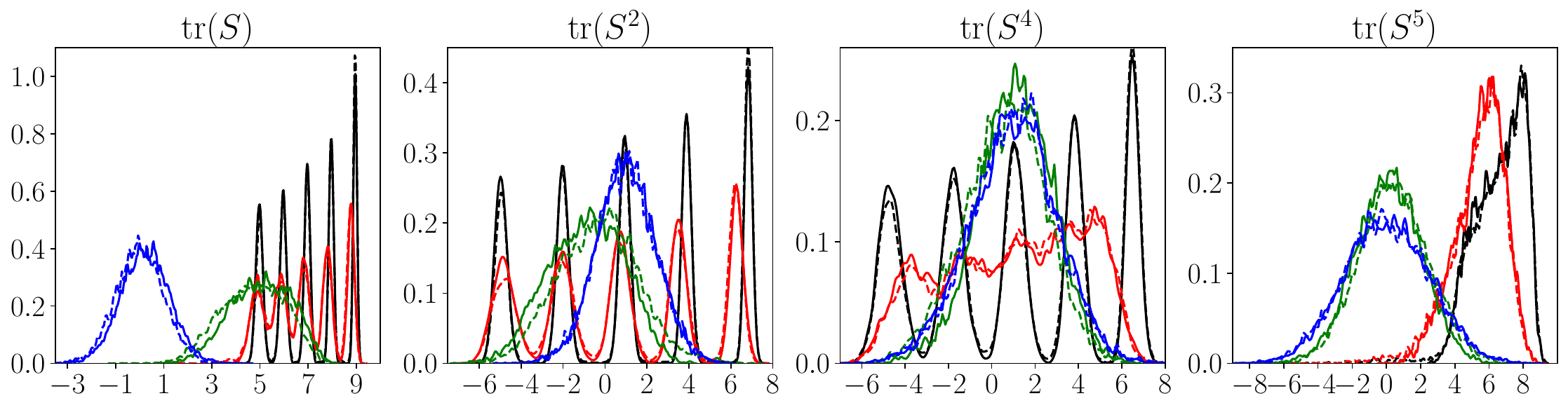}
\caption{Results for $\mathrm{SO}(10)$. 
Empirical densities of the statistics $\operatorname{tr}(S), \operatorname{tr}(S^2), \operatorname{tr}(S^4)$, and $\operatorname{tr}(S^5)$ for the forward process (solid line) and the learned reverse process (dashed line) at different steps $k = 0, 50, 200, 500$, colored in black, red, green, and blue, respectively.}\label{fig:SOn_result_5}
\end{figure}

Besides, to better demonstrate that our method can be applied to more general manifolds, we also consider the manifold $\varphi^{-1}(\mathrm{SO}(10))$, which is a nonlinear transformation of $\mathrm{SO}(10)$. Specifically, the manifold is defined as the space of matrices $S=(x_{i,j})_{1\leq i,j\leq 10}$ that satisfy the constraints $\sum_{k=1}^{10} \varphi(x_{i,k}) \varphi(x_{j,k}) -\delta_{i,j}=0$, $1\leq i\leq j\leq 10$, where $\varphi(x) = \tan(\frac{\pi}{4}x)$. Consequently, it is also a 45-dimensional submanifold of $\mathbb{R}^{100}$. The target distribution is again chosen as a multimodal distribution with five modes. The function $b$ in the forward Markov chain is defined as $b(x) = -\nabla V(x)$, where $V(x) = \frac{5}{2} \|x - I_{10}\|_{F}^2$ ($\|\cdot\|_F$ denotes the Frobenius norm). As shown in Figure \ref{fig:SOnTrans_result}, our method successfully generates data that matches the true data.

\begin{figure}[ht]
\centering
\includegraphics[width=0.95\linewidth]{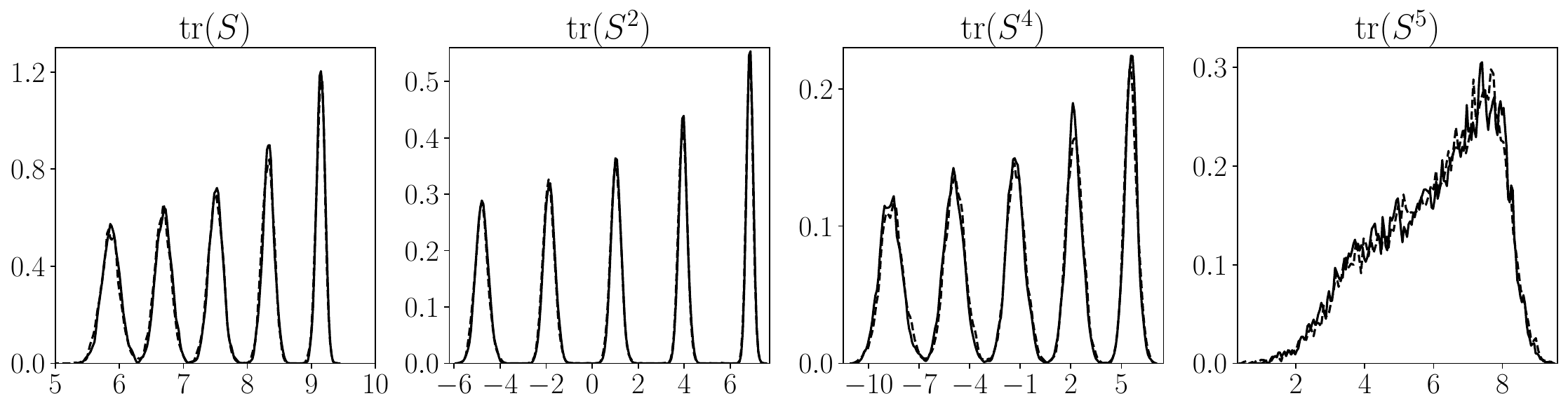}
\caption{Results for Transformed-$\mathrm{SO}(10)$. Empirical densities of the statistics  $\operatorname{tr}(S), \operatorname{tr}(S^2), \operatorname{tr}(S^4)$, and $\operatorname{tr}(S^5)$ for the samples generated by our method (dashed line) and samples from the target distribution (solid line).}\label{fig:SOnTrans_result}
\end{figure}

\subsection{Conserved Hamiltonian surface in phase space.}
We consider conserved Hamiltonian surface in phase space, where the manifold is defined as the space of states $x = (q, p)\in \mathbb{R}^{n_0}\times \mathbb{R}^{n_0}$ satisfying the constraint $\xi(q, p) = H(q, p) - E=0$, where $E>0$ is a constant, and the Hamiltonian is defined as $H(q,p)=\frac{|p|^2}{2 m}+U(q)$ with $U(q) = \frac{\kappa}{2}|q|^2 + \lambda \sum_{i=1}^{n_0} q_i^4$ representing the potential energy of a nonlinear oscillator.

In our experiment, we choose $n_0=10$, $E=10$, $m=0.5$, and $\kappa=\lambda=2$. Since the uniform distribution on the manifold is difficult to obtain, we define the function $b$ as $b(x)=-\nabla V(x)$, where $V(x) =\frac{5}{2} \sum_{i=1}^{n_0}(q_i^2 + (p_i-1)^2)$. We assume that each component $q_{i}$ follows a mixture of two Gaussian distributions. The resulting target distribution has in total $2^{n_0}$ modes.  Figure~\ref{fig:Hamilton_result} shows the empirical densities of components $q_1$, $q_2$, $p_1$, and $p_2$. The solid and dashed lines compare the distributions of the forward and the learned reverse processes at different steps $k$. The overlap of the black lines at $k=0$ shows that the generated samples match well with the target distribution.

\begin{figure}[htbp]
\centering
\includegraphics[width=0.95\linewidth]{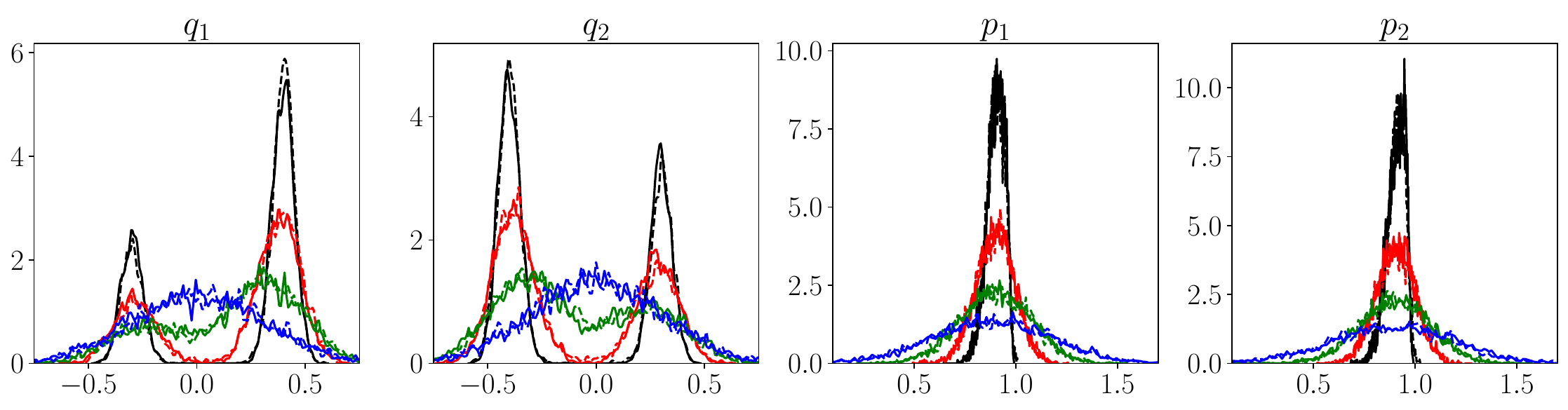}
\caption{Results for the conserved Hamiltonian surface. Empirical densities of $q_1$, $q_2$, $p_1$, and $p_2$ for the forward process (solid line) and the learned reverse process (dashed line) at different steps $k = 0, 20, 40, 150$, colored in black, red, green, and blue, respectively.}\label{fig:Hamilton_result}
\end{figure}

\subsection{Alanine dipeptide.}\label{subsec-ad}
We apply our method to alanine dipeptide, a commonly studied model system in bio-physics. The configuration of the system can be characterized by its two dihedral angles $\phi$ and $\psi$ (see Figure~\ref{dipeptide_system}). In this study, we are interested in the configurations  of the $10$ non-hydrogen atoms of the system (in $\mathbb{R}^{30}$) with the fixed angle $\phi=-70^\circ$.

In the forward process, $b$ is chosen as $-\nabla V$, where $V$ is proportional to the root mean squared deviation (RMSD) from a pre-selected  reference configuration $x^{\mathrm{ref}}$. Consequently, the prior distribution $p(x^{(N)})$ is a single-well distribution centered at $x^{\mathrm{ref}}$. Furthermore, we model $s^{(k+1),\theta}(x)$ in the reverse process using a network that preserves rotational equivariance and translational invariance. This design, together with our choice of $b$, guarantees that the distribution $p_{\theta}(x^{(0)})$ generated by our model is invariant under $\mathrm{SE}(3)~$\cite{xu2022geodiff}. We refer to Appendix~\ref{subsec-dipeptide_details} for further experimental details and to Appendix~\ref{Equivariant_NN} for theoretical support.

We employ three metrics to assess the quality of the generated configurations: the angle $\psi$, and two RMSDs (denoted by $\mathrm{RMSD}_1$ and $\mathrm{RMSD}_2$) with respect to two pre-defined reference configurations that are selected from two different wells. Figure~\ref{dipeptide_result} illustrates the empirical densities of these three metrics for the configurations generated by our model and the configurations in the dataset. The solid and dashed lines show the agreement between the distributions of the learned reverse process and the distributions of the forward process at different time steps $k$. In particular, the overlap between the lines in black, which correspond to step $k=0$, demonstrates that the distribution of the generated samples (dashed) closely matches the data distribution (solid).

\begin{figure}[htbp]
\centering
\begin{subfigure}[b]{0.28\textwidth}
\includegraphics[width=0.98\linewidth]{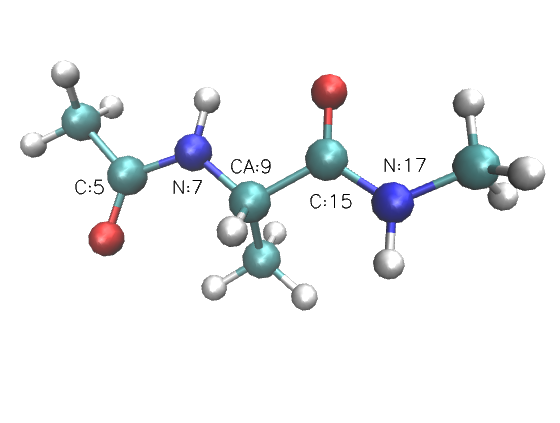}
\caption{Alanine dipeptide\label{dipeptide_system}}
\end{subfigure}
\begin{subfigure}[b]{0.70\textwidth}
\includegraphics[width=0.98\linewidth]{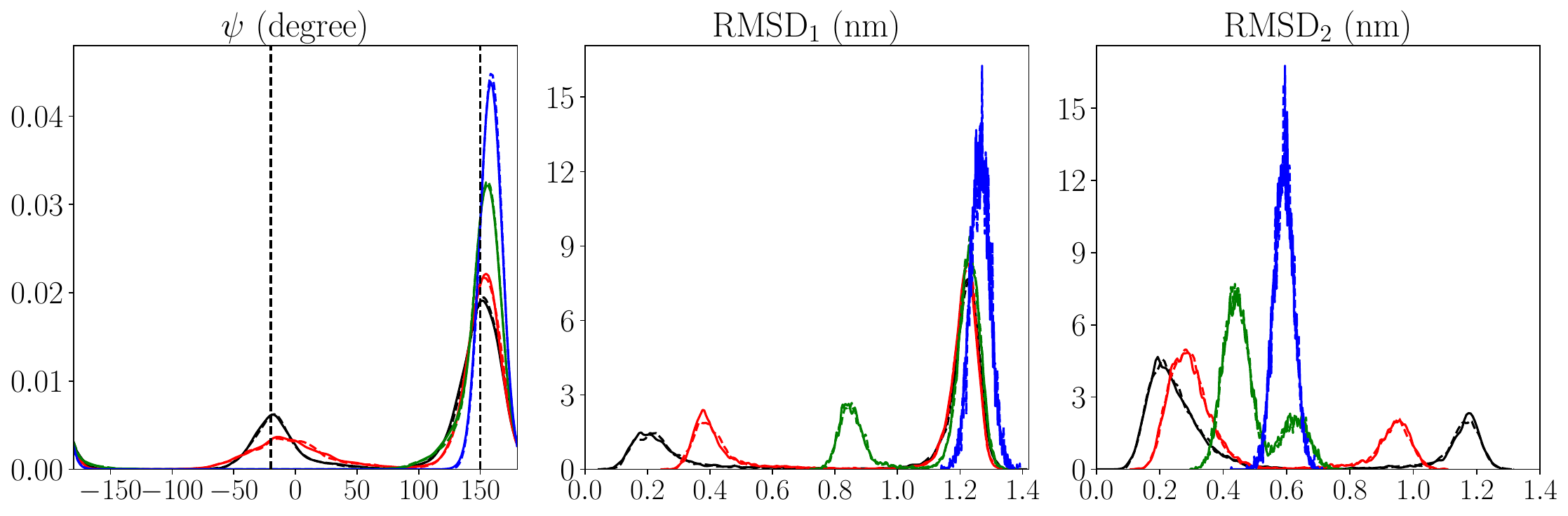}
\caption{Results for alanine dipeptide\label{dipeptide_result}}    
\end{subfigure}
\caption{
(a) Illustration of the system. Names and $1$-based indices are shown for atoms that define the dihedral angles. The dihedral angles $\phi$ and $\psi$ are defined by atoms whose $1$-based indices are $5,7,9, 15$ and $7,9, 15, 17$, respectively. (b) Empirical densities of the angle $\psi$, $\mathrm{RMSD}_1$, and $\mathrm{RMSD}_2$ for the forward process (solid line) and the learned reverse process (dashed line) at  steps $k = 0, 10, 40, 200$, colored in black, red, green, blue, respectively. The $\psi$ values of the two reference points that are used to define  $\mathrm{RMSD}_1$ and $\mathrm{RMSD}_2$ are $-20^\circ$ and $150^\circ$, respectively (as shown by the two vertical dashed lines in the left panel).
\label{dipeptide}}
\end{figure}

\section{Conclusion}
We have proposed Riemannian Denoising Diffusion Probabilistic Models for generative modeling on submanifolds. Our method does not rely on sophisticated geometric objects on manifold and it is applicable to high-dimensional manifolds with nontrivial geometry. We have provided a theoretical analysis of our method in the continuous-time limit, which elucidates its connection to Riemannian score-based generative models. We have demonstrated the strong capability of our method on manifolds from previous studies, as well as on those with complex geometries that can not be easily explored by existing methods.

\section*{Acknowledgment}
Zichen Liu acknowledges support from the China Scholarship Council under Grant No. 202306010047.
Wei Zhang is funded by the Eigene Stelle of Deutsche Forschungsgemeinschaft (DFG), project No. 524086759. 
Christof Sch{\"u}tte acknowledges funding by DFG under Germany's Excellence Strategy MATH+: The Berlin Mathematics Research Centre (EXC-2046/1, project No.\,390685689) and CRC 1114 ``Scaling Cascades in Complex Systems'' (project No.\,235221301). 
Tiejun Li acknowledges the support from National Key R\&D Program of China under grant 2021YFA1003301, and National Science Foundation of China under grant 12288101. 
This work is supported by High-performance Computing Platform of Peking University and Zuse Institute Berlin.

\appendix

\section{Proofs of the continuous-time limit} \label{app-sec-proofs}

In this section, we prove Theorem~\ref{thm-continuous-limit} and Corollary~\ref{corollary-continuous-limit-kl} in Section~\ref{sec-continuous-limit}. For simplicity of notation, we denote by $\partial_i$ the derivative with respect to $x_i$ in the ambient space, and by $I$ the identity matrix of order $n$. We use subscripts to denote components of a vector and entries of a matrix. Also recall that the orthogonal projection matrix $P(x)\in \mathbb{R}^{n\times n}$ is well defined for $x\in \mathbb{R}^n$ and has the expression 
\begin{equation}
P_{ij}(x) = \delta_{ij}-\sum_ {\alpha,\alpha'=1}^{n-d} \partial_i \xi_\alpha(x)  
(\nabla\xi^\top\nabla\xi)^{-1}_{\alpha\alpha'}(x) 
\partial_j \xi_{\alpha'}(x), \quad 1 \le i,j \le n\,,
\label{exp-pij}
\end{equation}
where $\delta_{ij}$ is the Dirac delta function. First, we present the proof of Theorem~\ref{thm-continuous-limit}.

\begin{proof}{\bf (Proof of Theorem~\ref{thm-continuous-limit}.)} Let us write the forward process in \eqref{projection-kstep-forward} as 
  \begin{equation*}
    x^{(k+1)} = x^{(k+\frac{1}{2})} + \nabla\xi(x^{(k)}) c(x^{(k+\frac{1}{2})})\,, 
  \end{equation*}
  where $x^{(k+\frac{1}{2})} = x^{(k)} + \sigma_k^2 b(x^{(k)}) + \sigma_k v^{(k)}$ and the dependence of $c$ on $x^{(k+\frac{1}{2})}$ is made explicit.
Applying Lemma~\ref{lemma-expansion-of-projection} at the end of this section, we obtain the expansion, for $1\le i\le n$,
\begin{align}\label{expansion-xk}
    & x^{(k+1)}_i \\
    = & x^{(k)}_i +\sum_{j=1}^{n} P_{ij}(x^{(k)}) \Big(\sigma_k^2 b_j(x^{(k)}) +
      \sigma_k v^{(k)}_j\Big) \notag\\
      &  + \frac{1}{2} \sum_{j,l,r,r'=1}^{n} 
      \Big((I-P)_{ir} P_{r'l} \,\partial_{r'} P_{rj}
      \Big)(x^{(k)})
      \Big(\sigma_k^2 b_j(x^{(k)}) + \sigma_k v^{(k)}_j\Big)\Big(\sigma_k^2
      b_{l}(x^{(k)}) + \sigma_k v^{(k)}_{l}\Big) \notag\\
      & + \frac{1}{6} \sum_{j,l,r=1}^{n} \sum_{\eta=1}^{n-d}
      \big(\partial_i \xi_\eta\, \partial_{jlr}^3
      c_\eta \big)(x^{(k)}) 
      \Big(
      x^{(k+\frac{1}{2})}_j-x^{(k)}_j\Big)
      \Big(
      x^{(k+\frac{1}{2})}_l-x^{(k)}_l\Big)
      \Big(
      x^{(k+\frac{1}{2})}_r-x^{(k)}_r\Big) \notag\\ 
      &  + o\big(|x^{(k+\frac{1}{2})}-x^{(k)}|^3\big) \notag\\
      =& x^{(k)}_i + \sigma_k  v^{(k)}_i + \sigma_k^2\sum_{j=1}^{n} P_{ij}(x^{(k)})  b_j(x^{(k)}) + \frac{\sigma_k^2}{2} 
\sum_{j,r,r'=1}^{n} 
      \big((I-P)_{ir}\, \partial_{r'} P_{rj}\big)(x^{(k)})  v^{(k)}_j v^{(k)}_{r'} \notag\\
      &+ \sigma_k^3 R^{(k)}_i + o(\sigma_k^3)\,,
\end{align}
where we have used the identity $\sum_{j=1}^n P_{ij}(x^{(k)}) v_j^{(k)}=v^{(k)}_i$ (since $v^{(k)}$ is a tangent vector), and $R_i^{(k)}$ is a term that satisfies $\sum_{i'=1}^n
  P_{ii'}(x^{(k)})R_{i'}^{(k)}=0$, for $1\le i \le n$.

  With the expansion above, we compute the loss function in 
  \eqref{training-loss}. Using \eqref{expansion-xk}, the relation
  $\beta_{k+1}=\sigma_k=\sqrt{h}g(kh)$, and the assumption that
  $s^{(k+1),\theta}(x^{(k+1)}) = s_\theta(x^{(k+1)}, (k+1)h)\in \mathbb{R}^n$, we can derive
\begin{align}
& \beta_{k+1}^2 \Big|P(x^{(k+1)})\Big(s^{(k+1),\theta}(x^{(k+1)}) - b(x^{(k+1)}) +
\frac{x^{(k+1)}-x^{(k)}}{\beta_{k+1}^2}\Big)\Big|^2 \notag \\
 =&  \sigma_{k}^2 \sum_{i=1}^n\bigg|\sum_{i'=1}^n
 P_{ii'}(x^{(k+1)})\Big[s_{\theta, i'}(x^{(k+1)}, (k+1)h) - b_{i'}(x^{(k+1)})
+ \sum_{j=1}^n P_{i'j}(x^{(k)}) b_j(x^{(k)})  \notag \\
&  \quad \quad   + \frac{1}{2} \sum_{j,r,r'=1}^n
\big((I-P)_{i'r}\,\partial_{r'} P_{rj}\big)(x^{(k)}) v^{(k)}_j v^{(k)}_{r'} 
+ \frac{ v^{(k)}_{i'}}{\sigma_k} + \sigma_k R^{(k)}_{i'} + o(\sigma_k)\Big]\bigg|^2 \notag \\
=& \mathcal{I}_1 + \mathcal{I}_2 + \mathcal{I}_3 + o(\sigma_k^2)\,,  \label{loss-term-as-sum-of-i123}
\end{align}
where the three terms on the last line are defined as 
\begin{align*}
    \mathcal{I}_1 := & \sigma_{k}^2 \sum_{i=1}^n\bigg|\sum_{i'=1}^n
    P_{ii'}(x^{(k+1)})\Big[s_{\theta,i'}(x^{(k+1)}, (k+1)h) - b_{i'}(x^{(k+1)}) + \sum_{j=1}^n P_{i'j}(x^{(k)}) b_j(x^{(k)})  \\
    & \quad \quad + \frac{1}{2} \sum_{j,r,r'=1}^n \big((I-P)_{i'r}\,\partial_{r'} P_{rj}\big)(x^{(k)}) v^{(k)}_j v^{(k)}_{r'} + \sigma_k R^{(k)}_{i'}
    \Big]\bigg|^2\,,\\
    \mathcal{I}_2 := & \sum_{i=1}^n\Big(\sum_{i'=1}^n P_{ii'}(x^{(k+1)}) v^{(k)}_{i'}\Big)^2\,,\\
\mathcal{I}_3: = &2\sigma_{k} \sum_{i,i'=1}^n
    P_{ii'}(x^{(k+1)})\bigg[s_{\theta,i'}(x^{(k+1)},(k+1)h) - b_{i'}(x^{(k+1)}) + \sum_{j=1}^n P_{i'j}(x^{(k)}) b_j(x^{(k)})  \\
    &  \quad \quad  + \frac{1}{2} \sum_{j,r,r'=1}^n \big((I-P)_{i'r}\,\partial_{r'}
    P_{rj}\big)(x^{(k)}) v^{(k)}_j v^{(k)}_{r'} + \sigma_k R^{(k)}_{i'}
    \bigg] v^{(k)}_{i}\,,
\end{align*}
respectively. In the following, we derive the expansions of the three terms above. For $\mathcal{I}_1$, expanding the functions $P, s_\theta, b$ using \eqref{expansion-xk}, we can derive 
\begin{align}
\mathcal{I}_1 =& \sigma_{k}^2 \sum_{i=1}^n\bigg|\sum_{i'=1}^n P_{ii'}(x^{(k)})\Big[s_{\theta,i'}(x^{(k)},kh) - b_{i'}(x^{(k)}) +
\sum_{j=1}^n P_{i'j}(x^{(k)}) b_j(x^{(k)}) \notag \\
& \quad \quad + \frac{1}{2} 
\sum_{j,r,r'=1}^n \big((I-P)_{i'r}\,\partial_{r'} P_{rj}\big)(x^{(k)}) v^{(k)}_j v^{(k)}_{r'}\Big] + o(1) \bigg|^2 \notag  \\
=& \sigma_{k}^2 \sum_{i=1}^n\Big|\sum_{i'=1}^n
P_{ii'}(x^{(k)})s_{\theta,i'}(x^{(k)},kh)\Big|^2 + o(\sigma_k^2)\,, \label{expansion-I1} 
\end{align}
where we have used the relations $P^2=P$ and $P(I-P)=0$ satisfied by the orthogonal projection matrix $P$ to derive the second equality. 
For $\mathcal{I}_2$, using the relation $P^2=P$ and \eqref{expansion-xk}, we can compute 
  \begin{align}
  \mathcal{I}_2 =& \sum_{i,i'=1}^n P_{ii'}(x^{(k+1)})
      v^{(k)}_{i}v^{(k)}_{i'} \notag\\
       =&\sum_{i,i'=1}^n P_{ii'}(x^{(k)}) v^{(k)}_{i}v^{(k)}_{i'} 
    + \sum_{i,i',r=1}^n \partial_r P_{ii'} (x^{(k)})
    \big(x^{(k+1)}_r - x^{(k)}_r\big) v^{(k)}_{i}v^{(k)}_{i'} \notag\\
    & + \frac{1}{2} \sum_{i,i'=1}^n\sum_{r,r'=1}^n \partial^2_{rr'}
    P_{ii'}(x^{(k)}) v^{(k)}_{i}v^{(k)}_{i'} 
    (x^{(k+1)}_r - x^{(k)}_r) (x^{(k+1)}_{r'} - x^{(k)}_{r'})  \notag  \\
    & + o(|x^{(k+1)} - x^{(k)}|^2)\notag \\
    =& |v^{(k)}|^2 + \sum_{i,i',r=1}^n \partial_r P_{ii'}(x^{(k)})
    \big(x^{(k+1)}_r - x^{(k)}_r\big) v^{(k)}_{i}v^{(k)}_{i'} \notag  \\
    &+ \frac{\sigma_k^2}{2} \sum_{i,i'=1}^n\sum_{r,r'=1}^n \partial^2_{rr'}
    P_{ii'}(x^{(k)}) v^{(k)}_{i}v^{(k)}_{i'} v^{(k)}_r v^{(k)}_{r'} + o(\sigma_k^2) \,.
      \label{expansion-I2-pre-1}
  \end{align}
Let's compute the three terms in \eqref{expansion-I2-pre-1}. Using the expression of $P_{ii'}$ in (\ref{exp-pij}), the fact that $\sum_{i=1}^n \partial_i \xi_\alpha(x^{(k)}) v^{(k)}_i = 0$, and the product rule, it is straightforward to verify that, for $1 \le r\le n$,
\begin{equation}\label{expansion-I2-pre-2}
  \sum_{i,i'=1}^n \partial_{r} P_{ii'} (x^{(k)}) v^{(k)}_{i}v^{(k)}_{i'} 
  = 
  -\sum_{i,i'=1}^n \partial_{r} 
  \Big(\sum_{\alpha,\alpha'=1}^{n-d} \partial_i \xi_\alpha  
(\nabla\xi^\top\nabla\xi)^{-1}_{\alpha\alpha'}
\partial_{i'} \xi_{\alpha'}\Big)
   (x^{(k)}) v^{(k)}_{i}v^{(k)}_{i'} = 0\,.
\end{equation} 
Similarly, we can verify that 
\begin{align}
&\sum_{i,i',r,r'=1}^n \partial^2_{rr'}
P_{ii'}(x^{(k)})\, v^{(k)}_{i}v^{(k)}_{i'} v^{(k)}_r v^{(k)}_{r'} \notag \\
=& -2\sum_{i,i',r,r'=1}^n\sum_{\alpha,\alpha'=1}^{n-d}\Big(\partial^2_{ir} \xi_\alpha (\nabla\xi^\top\nabla\xi)^{-1}_{\alpha\alpha'} \partial^2_{i'r'} \xi_{\alpha'}\Big)(x^{(k)}) \,v^{(k)}_{i}v^{(k)}_{i'} v^{(k)}_r v^{(k)}_{r'}\,. \label{expansion-I2-pre-3}
\end{align}
Hence, substituting \eqref{expansion-I2-pre-2} and \eqref{expansion-I2-pre-3} into \eqref{expansion-I2-pre-1}, we obtain
 \begin{align}
  \mathcal{I}_2 =&
   |v^{(k)}|^2 - \sigma_k^2
    \sum_{i,i',r,r'=1}^n\sum_{\alpha,\alpha'=1}^{n-d}\Big(\partial^2_{ir} \xi_\alpha (\nabla\xi^\top\nabla\xi)^{-1}_{\alpha\alpha'}\partial^2_{i'r'} \xi_{\alpha'}\Big)(x^{(k)}) \,v^{(k)}_{i}v^{(k)}_{i'} v^{(k)}_r v^{(k)}_{r'}
 + o(\sigma_k^2). \label{expansion-I2}
    \end{align}
For $\mathcal{I}_3$, we have
  \begin{align}
    \mathcal{I}_3 = & 2\sigma_{k} \sum_{i,i'=1}^n P_{ii'}(x^{(k)})\bigg[s_{\theta,i'}(x^{(k)},kh) - b_{i'}(x^{(k)}) + \sum_{j=1}^n P_{i'j}(x^{(k)}) b_j(x^{(k)}) \notag \\
    & \quad\quad + \frac{1}{2} \sum_{j,r,r'=1}^n
    \big((I-P)_{i'r}\, \partial_{r'} P_{rj}\big)(x^{(k)}) v^{(k)}_j v^{(k)}_{r'} + \sigma_k R^{(k)}_{i'} \bigg] v^{(k)}_{i}\notag\\
    & + 2\sigma^2_k\sum_{i,i',r=1}^n \partial_r (P_{ii'}(s_{\theta,i'} -
    b_{i'}))(x^{(k)},kh)\,  v^{(k)}_{r} v^{(k)}_{i} \notag\\
    & + 2\sigma^2_k\sum_{i,i',r,j=1}^n \partial_r P_{ii'}(x^{(k)}) 
    P_{i'j}(x^{(k)}) b_j(x^{(k)}) \,  v^{(k)}_{r} v^{(k)}_{i}\notag\\
    & + \sigma^2_k\sum_{i,i',r,j,r',j'=1}^n \partial_r P_{ii'}(x^{(k)}) 
    \big((I-P)_{i'j'}\, \partial_{r'} P_{j'j}\big)(x^{(k)}) v^{(k)}_j v^{(k)}_{r'}
    \,  v^{(k)}_{r} v^{(k)}_{i}+ o(\sigma_k^2)\notag\\
    =&2 \sigma_{k} 
\sum_{i,i'=1}^n P_{ii'}(x^{(k)})s_{\theta,i'}(x^{(k)},kh) v^{(k)}_{i} + 2\sigma^2_k \sum_{i,i',r=1}^n
    \partial_{r} (P_{ii'}(s_{\theta,i'} -
    b_{i'}))(x^{(k)},kh)  v^{(k)}_{r} v^{(k)}_{i} \notag\\
    & + 2\sigma^2_k\sum_{i,i',r,j=1}^n \partial_r P_{ii'}(x^{(k)}) 
    P_{i'j}(x^{(k)}) b_j(x^{(k)}) \,  v^{(k)}_{r} v^{(k)}_{i} \notag \\
    & + \sigma^2_k\sum_{i,i',r,j,r',j'=1}^n \partial_r P_{ii'}(x^{(k)}) 
    \big((I-P)_{i'j'}\, \partial_{r'} P_{j'j}\big)(x^{(k)}) v^{(k)}_j v^{(k)}_{r'}
    \,  v^{(k)}_{r} v^{(k)}_{i}+ o(\sigma_k^2)\,,  \label{expansion-I3-0}
  \end{align}
where we have used Taylor expansion with \eqref{expansion-xk} and the fact that $|s_{\theta}(x,(k+1)h) - s_{\theta}(x,kh)| =
O(h) = O(\sigma_k^2)$  to derive the first equality, and we have used the relations $P^2=P$, $P(I-P)=0$, and 
$\sum_{i'=1}^n P_{ii'}(x^{(k)})R^{(k)}_{i'}=0$ to derive the second equality. We further simplify the last two terms in the expression above. Notice that, similar to \eqref{expansion-I2-pre-2}, we can verify that 
\begin{equation}
\sum_{i,i',r=1}^n \partial_r P_{ii'}(x^{(k)}) 
    P_{i'j}(x^{(k)}) \, v^{(k)}_{r} v^{(k)}_{i} = 0\,, \quad 1 \le j \le n\,.\label{I3-pre-1}
\end{equation}
For the last term in \eqref{expansion-I3-0},  we can derive 
\begin{align}\label{I3-pre-2}
&\sum_{i,i',r,j,r',j'=1}^n \partial_r P_{ii'}(x^{(k)}) 
    \big((I-P)_{i'j'}\, \partial_{r'} P_{j'j}\big)(x^{(k)}) v^{(k)}_j v^{(k)}_{r'}
    \,  v^{(k)}_{r} v^{(k)}_{i} \notag \\
=    &\sum_{i',r,j,r',j'=1}^n 
    \big(\partial_r P_{i'j'}\, \partial_{r'} P_{j'j}\big)(x^{(k)}) v^{(k)}_j v^{(k)}_{r'}
    \,  v^{(k)}_{r} v^{(k)}_{i'} \notag \\
=    &\sum_{i',r,j,r',j'=1}^n      \bigg[\sum_{\alpha,\alpha',\eta,\eta'=1}^{n-d} \Big(\partial^2_{i'r} \xi_\alpha  
(\nabla\xi^\top\nabla\xi)^{-1}_{\alpha\alpha'}
\partial_{j'} \xi_{\alpha'}\Big)
    \Big( \partial^2_{jr'} \xi_\eta  
(\nabla\xi^\top\nabla\xi)^{-1}_{\eta\eta'}
\partial_{j'} \xi_{\eta'}\Big)\bigg](x^{(k)}) \notag \\
& \quad \quad \quad \quad \quad \cdot v^{(k)}_j v^{(k)}_{r'} v^{(k)}_{r} v^{(k)}_{i'} \notag \\
=    &\sum_{i',r,j,r'=1}^n      \bigg[\sum_{\alpha,\eta=1}^{n-d} \Big(\partial^2_{i'r} \xi_\alpha  
(\nabla\xi^\top\nabla\xi)^{-1}_{\alpha\eta}
 \partial^2_{jr'} \xi_\eta  
\Big)\bigg](x^{(k)})v^{(k)}_j v^{(k)}_{r'}
    \,  v^{(k)}_{r} v^{(k)}_{i'}\,,    
\end{align}
where the first equation follows by applying the product rule to the identity $P(I-P)=0$ and using the relation $\sum_{i=1}^n P_{ii'}v_i^{(k)}=v_{i'}^{(k)}$, the second equation follows from the expression (\ref{exp-pij}) and the fact that several terms vanish due to the orthogonality relation $\sum_{i=1}^n \partial_i \xi_\alpha(x^{(k)}) v_i^{(k)}=0$, and the last equation follows from the fact that $\sum_{j'=1}^n \partial_{j'}\xi_{\alpha'}\partial_{j'}\xi_{\eta'}=(\nabla\xi^\top \nabla\xi)_{\alpha'\eta'}$.
Combining \eqref{expansion-I3-0}, \eqref{I3-pre-1},
and \eqref{I3-pre-2}, we obtain 
\begin{align}
    \mathcal{I}_3
    = &
    2 \sigma_{k} 
\sum_{i,i'=1}^n P_{ii'}(x^{(k)})s_{\theta,i'}(x^{(k)},kh) v^{(k)}_{i} + 2\sigma^2_k \sum_{i,i',r=1}^n
    \partial_{r} (P_{ii'}(s_{\theta,i'} -
    b_{i'}))(x^{(k)},kh)  v^{(k)}_{r} v^{(k)}_{i} \notag\\
    &+\sigma^2_k\sum_{i',r,j,r'=1}^n      \bigg[\sum_{\alpha,\eta=1}^{n-d} \Big(\partial^2_{i'r} \xi_\alpha  
(\nabla\xi^\top\nabla\xi)^{-1}_{\alpha\eta}
 \partial^2_{jr'} \xi_\eta  
\Big)\bigg](x^{(k)})v^{(k)}_j v^{(k)}_{r'}
    \,  v^{(k)}_{r} v^{(k)}_{i'}+ o(\sigma_k^2)\,.\label{expansion-I3}
\end{align}
Substituting \eqref{expansion-I1}, \eqref{expansion-I2}, and \eqref{expansion-I3} into \eqref{loss-term-as-sum-of-i123},  we obtain (after cancellation of terms in $\mathcal{I}_2$ and $\mathcal{I}_3$)
\begin{align}
&  \beta_{k+1}^2 \bigg|P(x^{(k+1)})\Big(s^{(k+1),\theta}(x^{(k+1)}) -
b(x^{(k+1)}) + \frac{x^{(k+1)}-x^{(k)}}{\beta_{k+1}^2}\Big)\bigg|^2 \notag \\
=& \mathcal{I}_1 + \mathcal{I}_2 + \mathcal{I}_3 + o(\sigma_k^2)  \notag \\
= &
\sigma_{k}^2 \big|
P(x^{(k)})s_{\theta}(x^{(k)},kh)\big|^2 
+ \big|v^{(k)}\big|^2+ 2\sigma^2_k \sum_{i,i',r=1}^n
\partial_r (P_{ii'}(s_{\theta,i'} -
b_{i'}))(x^{(k)},kh)  v^{(k)}_{r} v^{(k)}_{i} \notag \\
& +2 \sigma_{k} \sum_{i,i'=1}^n P_{ii'}(x^{(k)})s_{\theta,i'}(x^{(k)},kh) v^{(k)}_{i} + o(\sigma_k^2)\,. \label{loss-term-expansion}
\end{align}

Now, we consider the terms on the right-hand side of the above expression in the limit $N\rightarrow +\infty$.
Using the fact that the forward process $x^{(k)}$ converges to the SDE~(\ref{sde-manifold})~\cite{projection_diffusion}
, we have 
\begin{equation*}
\lim_{N\rightarrow+\infty}\mathbb{E}_{\mathbb{Q}^{(N)}} \Big(\sum_{k=0}^{N-1}
    \sigma_{k}^2 \big|
    P(x^{(k)})s_{\theta}(x^{(k)},kh)\big|^2\Big)
    =\mathbb{E}_{\mathbb{Q}}\int_{0}^T \Big(\big|P(X_t)s_{\theta}(X_t, t)\big|^2\Big) g^2(t) dt\,.
    \end{equation*}
Since $v^{(k)}$ is the standard Gaussian random variable in $T_{x^{(k)}}\mathcal{M}$ confined in $\mathcal{F}_{x^{(k)}}^{(\sigma_k)}$ (the set of tangent vectors with which \eqref{projection-kstep-forward} has a solution), the set $\mathcal{F}_{x^{(k)}}^{(\sigma_k)}$ increases to the entire tangent space $T_{x^{(k)}}\mathcal{M}$ as $N\rightarrow +\infty$, we can show that (by bounded convergence theorem)
\begin{equation*}
\lim_{N\rightarrow+\infty}\mathbb{E}_{\mathbb{Q}^{(N)}}
\Big(\sigma_{k} \sum_{i,i'=1}^n P_{ii'}(x^{(k)})s_{\theta,i'}(x^{(k)},kh) v^{(k)}_{i}\Big)
=0, 
\end{equation*}
and
\begin{align*}
&\lim_{N\rightarrow+\infty}\mathbb{E}_{\mathbb{Q}^{(N)}}\Big(
\sigma^2_k \sum_{i,i',r=1}^n
    \partial_r (P_{ii'}(s_{\theta,i'} -
    b_{i'}))(x^{(k)},kh)  v^{(k)}_{r} v^{(k)}_{i}\Big)\\
=&\mathbb{E}_{\mathbb{Q}}\bigg[\int_0^T \Big(\sum_{i,i',r=1}^n \partial_r
    (P_{ii'}(s_{\theta,i'} - b_{i'}))(X_t,t) P_{ri}(X_t) \Big) g^2(t) dt\bigg]\\
    =&\mathbb{E}_{\mathbb{Q}}\bigg[\int_0^T \Big(\mbox{div}_{\mathcal{M}} \big(P(s_{\theta} - b)\big)(X_t,t)\Big) g^2(t) dt\bigg]\,,
\end{align*}
where the equality above can be verified using the fact that $v^{(k)}=P(x^{(k)}) z^{(k)}$ and $z^{(k)}$ converges to a standard Gaussian random variable in $\mathbb{R}^n$ as $N\rightarrow +\infty$.
 
Substituting \eqref{loss-term-expansion} into
\eqref{training-loss} and using the above equations, we can derive 
  \begin{align*}
  &\lim_{N\rightarrow+\infty} \Big(\mbox{Loss}^{(N)}(\theta)-\frac{1}{2}\mathbb{E}_{\mathbb{Q}^{(N)}}\sum_{k=0}^{N-1}\big|v^{(k)}\big|^2\Big)\\
    = & \mathbb{E}_{\mathbb{Q}}\int_{0}^T \bigg[\frac{1}{2}\big|P(X_t)s_{\theta}(X_t, t)\big|^2 + 
    \mbox{div}_{\mathcal{M}} \big(P(s_{\theta} - b)\big)(X_t,t)\bigg] g^2(t) dt\\
    = & \int_{0}^T \bigg[\int_{\mathcal{M}} \Big(\frac{1}{2}\big| P(x)s_{\theta}(x, t)\big|^2 + 
    \mbox{div}_{\mathcal{M}} \big(P(s_{\theta} - b)\big)(x,t)\Big) p(x,t)\, d\sigma_{\mathcal{M}}(x)\bigg]  g^2(t)\,dt\\
    =& \int_{0}^T \bigg[\int_{\mathcal{M}} \Big(\frac{1}{2}\big| P(x)s_{\theta}(x, t)\big|^2  
     -  (P(s_{\theta} - b))(x,t) \cdot \nabla_{\mathcal{M}} \log
     p(x,t)\Big) p(x,t)\, d\sigma_{\mathcal{M}}(x)\bigg]  g^2(t)\,dt  \\
    =& \frac{1}{2}\int_{0}^T \bigg[\int_{\mathcal{M}} \big| P(x)s_{\theta}(x,
    t)-\nabla_{\mathcal{M}} \log p(x,t)\big|^2 p(x,t)\,
    d\sigma_{\mathcal{M}}(x)\bigg]  g^2(t)\,dt 
    \\
    &+ \int_{0}^T \bigg[\int_{\mathcal{M}} \Big(\big(P(x)b(x)-\frac{1}{2}\nabla_{\mathcal{M}}
    \log p(x,t)\big)\cdot \nabla_{\mathcal{M}} \log p(x,t)\Big) p(x,t)\,d\sigma_{\mathcal{M}}(x)\bigg]  g^2(t)\,dt \\
    =& \mathbb{E}_{\mathbb{Q}}\bigg[\frac{1}{2}\int_{0}^T \big| P(X_t)s_{\theta}(X_t,
    t)-\nabla_{\mathcal{M}} \log p(X_t,t)\big|^2   g^2(t)\,dt \\
    &+ \int_{0}^T \Big(P(X_t)b(X_t)-\frac{1}{2}\nabla_{\mathcal{M}} \log
    p(X_t,t)\Big)\cdot \nabla_{\mathcal{M}} \log p(X_t,t) \, g^2(t) \,dt    \bigg]\,,
  \end{align*}
where we have used integration by parts on $\mathcal{M}$, and the expression $\mbox{div}_{\mathcal{M}} f=\sum_{i,r=1}^n P_{ir}\partial_r f_i$
for $f:\mathcal{M}\rightarrow \mathbb{R}^n$ (which can be verified using Lemma~A.1 in~\cite{zhang2020}). \hfill
\end{proof}

Next, we present the proof of Corollary~\ref{corollary-continuous-limit-kl}.

\begin{proof}{\bf (Proof of Corollary~\ref{corollary-continuous-limit-kl}.)} Using the assumption $\beta_{k+1}=\sigma_k$, the projection
  scheme in \eqref{projection-kstep-forward} and the relation $P(x^{(k)})\nabla\xi(x^{(k)})=0$, we can simplify the constant $C^{(N)}$ in \eqref{constant-c} as
\begin{equation}
C^{(N)}=-\mathbb{E}_{\mathbb{Q}^{(N)}} \Big(\log p(x^{(N)})+\frac{1}{2}\sum_{k=0}^{N-1} |v^{(k)}|^2
+ \sum_{k=0}^{N-1}\log \big(1-\epsilon_{x^{(k)}}^{(\sigma_k)}\big)
\Big).
\end{equation}
Therefore, using 
the definition of relative entropy (see \eqref{relative-entropy-formulation-of-variational-bound}), the loss function in \eqref{training-loss}, the constant $C^{(N)}$ in \eqref{constant-c}, and 
applying Theorem~\ref{thm-continuous-limit}, we have 
\begin{align}\label{proof-corolloary-1}
    & \lim_{N\rightarrow +\infty}
H\big(\overleftarrow{\mathbb{Q}}^{(N)}\,|\,\mathbb{P}^{(N)}_\theta\big) \notag \\
    =& \lim_{N\rightarrow +\infty} \left[\mathrm{Loss}^{(N)}(\theta) + C^{(N)} + \mathbb{E}_{\mathbb{Q}^{(N)}}\Big(\sum_{k=0}^{N-1}\log \big(1-\epsilon_{x^{(k+1)},\theta}^{(\sigma_k)}\big)\Big)\right] + \mathbb{E}_{q_0}(\log q_0)  \notag \\
    =& \lim_{N\rightarrow +\infty} \left[\mathrm{Loss}^{(N)}(\theta) -\mathbb{E}_{\mathbb{Q}^{(N)}} \left(\log p(x^{(N)})
    +\frac{1}{2}\sum_{k=0}^{N-1} |v^{(k)}|^2
+ \sum_{k=0}^{N-1}\log \frac{1-\epsilon_{x^{(k)}}^{(\sigma_k)}}{1-\epsilon_{x^{(k+1)},\theta}^{(\sigma_k)}}\right)\right]  \notag  \\
& + \mathbb{E}_{q_0}(\log q_0)  \notag \\
=& \lim_{N\rightarrow +\infty} \Big(\mathrm{Loss}^{(N)}(\theta) -\mathbb{E}_{\mathbb{Q}^{(N)}} \log p(x^{(N)})
-\frac{1}{2}\sum_{k=0}^{N-1} |v^{(k)}|^2 \Big) + \mathbb{E}_{q_0}(\log q_0) \notag  \\
=& \mathbb{E}_{\mathbb{Q}} \bigg[\log p(X_0,0) -\log p(X_T,T) + \frac{1}{2}\int_{0}^T \big|
P(X_t)s_{\theta}(X_t, t)-\nabla_{\mathcal{M}} \log p(X_t,t)\big|^2 g^2(t)\, dt  \notag \\
    & + \int_{0}^T \Big(P(X_t)b(X_t)-\frac{1}{2}\nabla_{\mathcal{M}} \log
    p(X_t,t)\Big)\cdot \nabla_{\mathcal{M}} \log p(X_t,t)\,  g^2(t)\, dt
    \bigg] \,, 
\end{align}
where the third equality follows because the terms containing $\epsilon_{x^{(k)}}^{(\sigma_k)}$ and $\epsilon_{x^{(k+1)},\theta}^{(\sigma_k)}$ converge to zero sufficiently fast as $N\rightarrow +\infty$. Note that the density $p(x,t)$ of SDE~(\ref{sde-manifold}) solves the Fokker-Planck equation
\begin{equation}
  \frac{\partial p}{\partial t}(x,t) = -g^2(t) \mbox{div}_{\mathcal{M}} \big(P(x) b(x) p(x,t)\big)  + \frac{g^2(t)}{2} \Delta_{\mathcal{M}} p(x,t)\,, \quad x\in \mathcal{M},\quad t\in[0,T]\,.
  \label{fp}
\end{equation}
Therefore, we have 
\begin{align}\label{proof-corolloary-2}
  & \mathbb{E}_{\mathbb{Q}} \Big[\log p(X_0,0) -\log p(X_T,T)\Big] \notag \\
   = &\int_{\mathcal{M}} \log p(x,0) p(x,0)\, d\sigma_{\mathcal{M}}(x) -
    \int_{\mathcal{M}} \log p(x,T) p(x,T)\, d\sigma_{\mathcal{M}}(x)  \notag  \\
    = & -\int_0^T \frac{d}{dt} \Big(\int_{\mathcal{M}} \log p(x,t)\, p(x,t)\, d\sigma_{\mathcal{M}}(x)\Big) \,dt  \notag \\
    = & -\int_0^T \bigg[\int_{\mathcal{M}} \big(\log p(x,t) + 1\big) \frac{\partial p}{\partial t} (x,t)\, d\sigma_{\mathcal{M}}(x)\bigg] dt  \notag \\
    = & -\int_0^T \bigg[\int_{\mathcal{M}} \big(\log p(x,t) + 1\big)
    \Big(- \mbox{div}_{\mathcal{M}} \big(P(x) b(x) p(x,t)\big)  + \frac{1}{2} \Delta_{\mathcal{M}} p(x,t)\Big) g^2(t) d\sigma_{\mathcal{M}}(x)\bigg] dt  \notag \\
    = & -\int_0^T \bigg[\int_{\mathcal{M}}  \Big(\big(P(x) b(x)  -
    \frac{1}{2}\nabla_{\mathcal{M}} \log p(x,t)\big) \cdot \nabla_{\mathcal{M}} \log p(x,t)\Big)\, p(x,t)\, d\sigma_{\mathcal{M}}(x)\bigg] g^2(t) \,dt  \notag \\
    =& -\mathbb{E}_{\mathbb{Q}} \bigg[\int_{0}^T \Big(P(X_t)b(X_t)-\frac{1}{2}\nabla_{\mathcal{M}} \log
    p(X_t,t)\Big)\cdot \nabla_{\mathcal{M}} \log p(X_t,t) \, g^2(t) \,dt\bigg] \,,  
\end{align}
where we have used \eqref{fp} to derive the fourth equality and integration by parts on $\mathcal{M}$ to derive the fifth equality. 
  Combining \eqref{proof-corolloary-1} and \eqref{proof-corolloary-2}, we obtain
\begin{equation}
     \lim_{N\rightarrow +\infty}
H\big(\overleftarrow{\mathbb{Q}}^{(N)}\,|\,\mathbb{P}^{(N)}_\theta\big) 
    = \mathbb{E}_{\mathbb{Q}} \bigg[\frac{1}{2}\int_{0}^T \big| P(X_t)s_{\theta}(X_t, t)-\nabla_{\mathcal{M}} \log p(X_t,t)\big|^2  g^2(t)\, dt\,\bigg]\,. 
  \label{proof-corrolary-part-1}
\end{equation}

Finally, note that $\overleftarrow{\mathbb{Q}}$ is the path measure of the time-reversal $Y_t=X_{T-t}$ of SDE~(\ref{sde-manifold}), which satisfies~\cite[Theorem~3.1]{de2022riemannian}
\begin{equation}
  dY_t =  g^2(T-t) \Big(-P(Y_t) b(X_t)+ \nabla_{\mathcal{M}} \log p(Y_t, T-t)\Big)\,dt + g(T-t) dW^\mathcal{M}_t\,, \quad t\in[0,T]\,,
  \label{sde-manifold-reversal}
\end{equation}
and $\mathbb{P}_\theta$ is the path measure of SDE~(\ref{sde-manifold-reversed-theta}).
Applying Girsanov's theorem~\cite[Theorem~8.1.2]{analysis_manifold}, we obtain
\begin{align}
\frac{d \mathbb{P}_\theta}{d \overleftarrow{\mathbb{Q}}} = \exp\Big(&\int_0^T g^2(T-t) \big(P(Y_t) s_\theta(Y_t,T-t)-\nabla_{\mathcal{M}} \log p(Y_t, T-t)\big)\cdot dW_t^\mathcal{M} \notag \\
& - \frac{1}{2}\int_0^T \big|P(Y_t)s_\theta(Y_t,T-t)-\nabla_{\mathcal{M}} \log p(Y_t, T-t)\big|^2 g^2(T-t)\,dt\Big)\,, \label{girsanov}
\end{align}
where $W_t^\mathcal{M}$ is a Brownian motion on $\mathcal{M}$ under $\overleftarrow{\mathbb{Q}}$. Therefore, we have
\begin{align}
H(\overleftarrow{\mathbb{Q}}\,|\,\mathbb{P}_\theta) = &
\mathbb{E}_{\overleftarrow{\mathbb{Q}}}\Big(\log \frac{d\overleftarrow{\mathbb{Q}}}{d\mathbb{P}_\theta}\Big) \notag \\
= & \mathbb{E}_{\overleftarrow{\mathbb{Q}}}\Big(
\frac{1}{2}\int_0^T \big|P(Y_t)s_\theta(Y_t,T-t)-\nabla_{\mathcal{M}} \log p(Y_t, T-t)\big|^2 g^2(T-t)\,dt
\Big) \notag \\
= & \mathbb{E}_{\mathbb{Q}}\Big(
\frac{1}{2}\int_0^T \big|P(X_t) s_\theta(X_t,t)-\nabla_{\mathcal{M}} \log p(X_t, t)\big|^2 g^2(t)\,dt
\Big)\,, \label{proof-corrolary-part-2}
\end{align}
where the second equality follows from the fact that the stochastic integration in \eqref{girsanov} vanishes after taking logarithm and  expectation, and the third equality follows by a change of variable $t\leftarrow T-t$ and the fact that $Y_t=X_{T-t}$. The conclusion is obtained after combining \eqref{proof-corrolary-part-1} and \eqref{proof-corrolary-part-2}. \hfill
\end{proof}

Finally, we present the technical lemma 
on the projection scheme in \eqref{projection}, which  was used in the proof of Theorem~\ref{thm-continuous-limit}.

\begin{lemma}
Given $x\in\mathcal{M}$ and $x'\in\mathbb{R}^n$, the solution to the problem
\begin{equation}
  y = x' + \nabla\xi(x) c(x'), \quad c(x')\in \mathbb{R}^{n-d}, ~~\mbox{s.t.}~~\xi(y) =  0 
  \label{projection-repeat}
\end{equation}
has the following two expansions as $x'$ approaches $x$ 
\begin{align}
  \partial_j c_\eta (x) &= - \sum_{\alpha=1}^{n-d}(\nabla\xi^\top \nabla\xi)^{-1}_{\eta\alpha}(x) \partial_j
  \xi_\alpha(x)\,, \quad 1 \le j \le n\,, \label{lagrange-multiplier-diff1} \\
\partial^2_{jl} c_\eta(x) &= 
  \sum_{\alpha=1}^{n-d}(\nabla\xi^\top\nabla\xi)^{-1}_{\eta\alpha}(x)
\sum_{r,r'=1}^{n}\Big(\partial_r\xi_\alpha \partial_{r'}
  P_{rj}\, P_{r'l} \Big)(x)\,, \quad 1 \le j,l \le n\,, \label{lagrange-multiplier-diff2}
\end{align}
for $1\le \eta \le n-d$. Moreover, as $x'$ approaches $x$, the following expansion of $y$ in \eqref{projection-repeat} holds
\begin{align}
y_i =& x_i + \sum_{j=1}^n P_{ij}(x) (x'_j - x_j) + \frac{1}{2}
\sum_{j,l=1}^n\bigg[\sum_{r,r'=1}^n \big((I-P)_{ir} P_{r'l}\partial_{r'}
P_{rj} \big)(x)\bigg] (x'_j -
x_j)(x'_{l} - x_{l}) \notag \\
& +  \frac{1}{6}
\sum_{j,l,r=1}^n
\bigg(
\sum_{\eta=1}^{n-d}
\partial_i \xi_\eta(x)\,
\partial^3_{jlr} c_\eta(x)
\bigg)   (x'_j - x_j)(x'_{l} - x_{l})(x'_{r} - x_{r}) + o(|x'-x|^3)\,, \label{constraint-expansion}
\end{align}
where $1\le i \le n$.
\label{lemma-expansion-of-projection}
\end{lemma}

\begin{proof}
Differentiating (with respect to $x'$) the constraint equation 
\begin{equation*}
\xi_\alpha(x' + \nabla\xi(x)c(x')) = 0\,, \quad \alpha = 1,\dots, n-d\,,
\end{equation*}
we get 
\begin{equation}
\sum_{r=1}^n
      \partial_r \xi_\alpha \big(x'+\nabla\xi(x)c(x')\big)
      \Big(\delta_{rj} + \sum_{\eta=1}^{n-d}
      \partial_r \xi_\eta(x)\partial_{j} c_\eta(x')\Big) = 0\,, \quad 1 \le j \le n\,.
      \label{1st-derivative-constraint-eqn}
  \end{equation}
Setting $x'=x$ in \eqref{1st-derivative-constraint-eqn} (notice that $c(x')=0$ when $x'=x$) and multiplying both
sides by $(\nabla\xi^\top \nabla\xi)^{-1}(x)$, we obtain \eqref{lagrange-multiplier-diff1}. In particular, using \eqref{exp-pij}, we have 
\begin{equation}
\delta_{rj} + \sum_{\eta=1}^{n-d}\partial_r \xi_\eta(x)\partial_j c_\eta(x) = 
    \delta_{rj} - \sum_{\eta,\alpha=1}^{n-d}\Big(\partial_r
    \xi_\eta (\nabla \xi^\top\nabla\xi)^{-1}_{\eta\alpha}
    \partial_j \xi_\alpha\Big)(x) = P_{rj}(x)\,, \quad 1\le r,j\le n\,.
    \label{1st-derivative-constraint-eqn-p}
  \end{equation}
  Next, we show \eqref{lagrange-multiplier-diff2}.  Differentiating \eqref{1st-derivative-constraint-eqn} again, setting $x'=x$
  and using \eqref{1st-derivative-constraint-eqn-p}, we get, for $1\le \alpha \le n-d$ and $1\le j,l\le n$,
    \begin{align*}
      0 =& \sum_{r,r'=1}^n \partial^2_{rr'} \xi_\alpha(x) \Big(\delta_{rj} + 
      \sum_{\eta=1}^{n-d} \partial_r \xi_\eta(x)\partial_j c_\eta(x)\Big) \Big(\delta_{r'l} +
\sum_{\eta=1}^{n-d} \partial_{r'} \xi_\eta(x)\partial_l
      c_\eta(x)\Big) \\
& + \sum_{r=1}^n \partial_r \xi_\alpha(x) \Big(\sum_{\eta=1}^{n-d} \partial_r \xi_\eta(x)\partial^2_{jl} c_\eta(x)\Big) \\
      = &\sum_{r,r'=1}^n \Big(\partial^2_{rr'} \xi_\alpha P_{rj} P_{r'l}\Big)(x)
      + \sum_{\eta=1}^{n-d}\Big((\nabla \xi^\top \nabla\xi)_{\alpha\eta}
      \partial^2_{jl} c_\eta\Big)(x)\,,
    \end{align*}
from which we can solve, for $1\le \eta \le n-d$ and $1\le j,l\le n$, 
\begin{align}
\partial^2_{jl} c_\eta(x) = &
- \sum_{\alpha=1}^{n-d} \sum_{r,r'=1}^{n} \Big((\nabla\xi^\top\nabla\xi)^{-1}_{\eta\alpha} \partial^2_{rr'} \xi_\alpha P_{rj}P_{r'l}\Big)(x) \notag \\
= & \sum_{\alpha=1}^{n-d} \sum_{r,r'=1}^{n} \Big((\nabla\xi^\top\nabla\xi)^{-1}_{\eta\alpha}\, \partial_r \xi_\alpha \partial_{r'} P_{rj}\,
P_{r'l}\Big)(x)
\,, \notag 
\end{align}
where the second equality follows from the product rule and the identity
$\sum_{r=1}^{n} P_{rj} \partial_r \xi_{\alpha}=0$.
This shows \eqref{lagrange-multiplier-diff2}.

Lastly, we prove the expansion in \eqref{constraint-expansion}. Note that \eqref{lagrange-multiplier-diff2} and \eqref{exp-pij} implies 
\begin{equation}
\sum_{\eta=1}^{n-d} (\partial_i \xi_\eta\partial^2_{jl} c_{\eta})(x)  = 
\sum_{r,r'=1}^{n} \big((I-P)_{ir} P_{r'l}\,\partial_{r'} P_{rj}\big)(x) \,, \quad 1 \le i,j,l \le n\,.
\label{2nd-derivative-constraint-eqn-p}
  \end{equation}
  By expanding $c(x')$ at $x'=x$ to the third order, noticing that $c(x)=0$,
  and using \eqref{1st-derivative-constraint-eqn-p} and \eqref{2nd-derivative-constraint-eqn-p}
  for the first and second order derivatives respectively,
  we can derive
    \begin{align*}
      y_i =& x'_i + \sum_{\eta=1}^{n-d} \partial_i \xi_\eta(x) c_\eta(x') \\
      =& x_i + (x'_i-x_i) + \sum_{\eta=1}^{n-d} \partial_i \xi_\eta(x)
      \bigg[\sum_{j=1}^{n} \partial_j c_{\eta}(x) (x'_j - x_j) +
    \frac{1}{2} \sum_{j,l=1}^{n} \partial^2_{jl} c_{\eta}(x) (x'_j -
      x_j)(x'_{l} - x_{l})  \\
      & + \frac{1}{6}\sum_{j,l,r=1}^{n}  \partial^3_{jlr} c_\eta(x)\, (x'_j -
      x_j)(x'_{l} - x_{l})(x'_{r} - x_{r})\bigg] + o(|x'-x|^3) \\
      =& x_i + \sum_{j=1}^{n} P_{ij}(x) (x'_j - x_j) + \frac{1}{2}
      \sum_{j,l=1}^{n} \bigg[\sum_{r,r'=1}^{n} \Big((I-P)_{ir} \partial_{r'}
      P_{rj} P_{r'l}\Big)(x)\bigg] (x'_j - x_j)(x'_{l} - x_{l}) \\
      & +  \frac{1}{6} \sum_{j,l,r=1}^{n} 
      \bigg(\sum_{\eta=1}^{n-d} 
      \partial_i \xi_\eta(x)
      \partial^3_{jlr} c_\eta(x)\bigg) (x'_j -
      x_j)(x'_{l} - x_{l})(x'_{r} - x_{r}) + o(|x'-x|^3)\,,
    \end{align*}
  which proves \eqref{constraint-expansion}.  \hfill
\end{proof}

\section{Details of algorithms and experiments}\label{Exp_details}

\subsection{Neural networks and training setup.}
As described in Theorem \ref{thm-continuous-limit}, the functions $(s^{(k+1),\theta}(x))_{0\le k\le N-1}$ are represented by a single function $s_{\theta}(x, t)$ with parameter $\theta$, which is in turn modeled by a multilayer perceptron (MLP). We employ SiLU as the activation function.  We do not require that the output of the neural network belongs to the tangent space, thanks to the presence of the projection in both the forward and the reverse processes. Alternative strategies for designing neural networks with outputs in tangent space are proposed in \cite{de2022riemannian}.

Moreover, we define $g(t)$ as $g(t) = \gamma_{\mathrm{min}} + \frac{t}{T}(\gamma_{\mathrm{max}} - \gamma_{\mathrm{min}})$, where $\gamma_{\mathrm{max}} \geq \gamma_{\mathrm{min}} > 0$. The parameters in the Markov chain are chosen as $\sigma_k = \beta_{k+1} = \sqrt{h}g(kh)$, with $h = \frac{T}{N}$ and $k = 0, 1, \dots, N-1$.

We train our models using PyTorch, where we employ the Adam optimizer with fixed learning rate $r=5\times 10^{-4}$ and we clip the gradients of the parameters when the $2$-norm exceeds $10$. We also implement an exponential moving average for the model weights \cite{polyak1992acceleration} with a decay rate of $0.999$. All experiments are run on a single NVIDIA A40 GPU with 48G memory. In each run, the dataset is divided into training, validation, and test sets with ratio $80$:$10$:$10$. Values of all the parameters in our experiments are summarized in Table~\ref{Hyperparameters}.

\begin{table}[htbp]
\centering
\begin{tabular}{lccccccccc}
\toprule
Datasets & $\gamma_{\mathrm{min}}$ & $\gamma_{\mathrm{max}} $& $N$ & $T$ &$l_{\mathrm{f}}$ & $N_{\mathrm{epoch}}$& $B$  & $N_{\mathrm{n}}$& $N_{\mathrm{l}}$ \\
\midrule
Bunny, $k=50$& $0.07$& $0.07$& $800$& $8.0$&$100$ & $2000$ &$2048$  & $256$ & $5$ \\ 
Bunny, $k=100$& $0.07$ & $0.07$& $500$&$5.0$ & $100$&  $2000$ &$2048$   & $256$ & $5$  \\ 
Spot, $k=50$ & $0.1$&$0.1$ & $500$ & $5.0$ &$100$ &  $2000$ &$2048$  & $256$ & $5$   \\ 
Spot, $k=100$ & $0.1$& $0.1$& $300$& $3.0$ &$100$ &  $2000$ &$2048$   & $256$ & $5$  \\ 
\midrule
$\mathrm{SO}(10)$ & $0.2$ & $2.0$& $500$& $1.0$ & $100$ &$2000$ & $512$ & $512$ & $3$ \\ 
Transformed-$\mathrm{SO}(10)$ & $0.2$ & $1.3$ & $250$ & $2.5$ & $100$ & $1000$ & $512$ & $256$ & $3$\\
\midrule
Hamiltonian & $0.1$ & $1.5$ & $150$ & $1.5$ & $1$ & $4000$ & $512$ & $256$ & $3$  \\
\midrule
Alanine dipeptide & $1.0$ & $1.0$ & $200$&$0.1$ & $100$& $5000$ & $512$ & $512$ & $5$ \\ 
\bottomrule
\end{tabular}
\caption{Parameters in our experiments. $\gamma_{\mathrm{min}},\gamma_{\mathrm{max}},N,T$ are the parameters in our model; $l_{\mathrm{f}},N_{\mathrm{epoch}},B$ are the parameters in Algorithm \ref{algo_training}; $N_{\mathrm{n}},N_{\mathrm{l}}$ are the numbers of the hidden nodes per layer and the hidden layers of the neural networks, respectively.}\label{Hyperparameters}
\end{table}

\subsection{Mesh data on learned manifolds.}\label{Bunny_Spot_details}
The function $\xi:\mathbb{R}^3\rightarrow\mathbb{R}$ is modeled by a MLP with $3$ hidden layers, each of which has $128$ nodes. Different from the activation function in our model, here we use Softplus activation function, where the parameter $\beta$ is set to $10$. The loss function for learning $\xi$ is 
\begin{equation}\label{sdf_loss}
\ell(\xi)=\frac{1}{|\mathcal{D}|} \sum_{x\in \mathcal{D}}|\xi(x)|+ \frac{\lambda}{|\mathcal{D}'|}\sum_{y \in \mathcal{D}'}(|\nabla \xi(y)|-1)^2,
\end{equation}
where $\lambda = 0.1$, $\mathcal{D}$ denotes the set of vertices of a high-resolution mesh, and the set $\mathcal{D}'$ contains samples near the manifolds that are obtained by perturbing samples $x\in \mathcal{D}$ according to $y=x+c\epsilon$, with $\epsilon\sim \mathcal{N}(0,~I_3)$ and $c=0.05$.
The first term in \eqref{sdf_loss} imposes that $\xi$ is close to zero on vertices, whereas the second term serves as a regularization term and ensures that $\xi$ has non-vanishing gradient near the manifold. 
The neural network is trained for $200000$ steps using Adam optimizer, with batch size $512$ and learning rate $10^{-4}$. 

With the learned function $\xi$, we consider the manifold defined by 
$\mathcal{M}=\{x\in \mathbb{R}^3 | \xi(x)=0\}$. The values of $\xi$ on the dataset are at the order $10^{-2}$. To ensure that the data is on $\mathcal{M}$ with high precision,  we refine the dataset by solving the following ordinary differential equation (ODE):
\begin{equation}\label{refine_ode}
\frac{d x_t}{dt}=-\xi(x_t) \nabla\xi(x_t),  \quad t\ge 0\,,
\end{equation}
starting from each point in the dataset until the condition $|\xi(x_t)|<10^{-5}$ is reached (notice that 
\eqref{refine_ode} is a gradient flow and $\lim_{t\rightarrow \infty} |\xi(x_t)|=0$). This ensures that the refined points conform to the manifold accurately.

\subsection{High-dimensional special orthogonal group.}\label{SOn_details}
The dataset is constructed as a mixture of $5$ wrapped normal distributions, each of which is the image (under the exponential map) of a normal distribution in the tangent space of a center $S_i\in \mathrm{SO}(10)$, $1\le i\le 5$. 
To ensure multimodality, we define the centers $S_i$ as follows. We initially define a $2\times 2$ matrix $A_0 := \begin{bmatrix}\cos\frac{\pi}{3} & \sin\frac{\pi}{3} \\-\sin\frac{\pi}{3} & \cos\frac{\pi}{3}\end{bmatrix}$, which represents a rotation by $\frac{\pi}{3}$ radians. We then construct block diagonal matrices of order $10$ by incorporating $A_0$ and the identity matrix $I_2$ in various combinations:
\begin{align}
&X_1=\operatorname{diag}\{A_0, I_2, I_2, I_2, I_2\},  ~~ 
X_2 =\operatorname{diag}\{A_0, A_0, I_2, I_2, I_2\}, ~~  
X_3=\operatorname{diag}\{A_0, A_0, A_0, I_2, I_2\}, \notag \\
&\hspace{2cm} X_4=\operatorname{diag}\{A_0, A_0, A_0, A_0, I_2\},  ~~ 
X_5=\operatorname{diag}\{A_0, A_0, A_0, A_0, A_0\}. \label{SO10_center}
\end{align}

The centers $S_i$ of the wrapped normal distributions are chosen as $S_i=Q_i^\top X_i Q_i$, where $Q_i \in \mathrm{SO}(10)$ are randomly drawn from the uniform distribution. According to \eqref{SO10_center},  the statistics $\eta(S)=(\operatorname{tr}(S), \operatorname{tr}(S^2), \operatorname{tr}(S^4), \operatorname{tr}(S^5))$ of the centers can be explicitly computed (using the trace identities $\operatorname{tr}(AB)=\operatorname{tr}(BA)$ and $\operatorname{tr}(Q_i^\top X_iQ_i)=\operatorname{tr}(X_i)$) as
\begin{align}
& \eta(S_1)=(9,7,7,9), \quad
\eta(S_2)=(8,4,4,8),  \quad
\eta(S_3)=(7,1,1,7),  \notag \\
& \hspace{1cm} \eta(S_4)=(6,-2,-2,6), \quad
\eta(S_5)=(5,-5,-5,5). \label{SO10_center_eta}
\end{align}

To generate data in the dataset, we select a center $S_i$ with equal probability, sample tangent vectors $Y$ from the normal distribution (in the tangent space at $S_i$) with zero mean and standard deviation $0.05$, and then compute their images $S$ under the exponential map, that is, $S=S_i\mathrm{e}^{S_i^\top Y}$.

\subsection{Alanine dipeptide.}\label{subsec-dipeptide_details}
To generate the dataset, we initially perform a constrained molecular simulation of alanine dipeptide in water for $1$ns using the molecular dynamics package GROMACS~\cite{GROMACS} with step size $1$fs. We apply the harmonic biasing method in COLVARS module~\cite{Fiorin2013}, where the collective variable is chosen as the dihedral angle $\phi$ and the harmonic potential is centered at $\phi=-70^\circ$ with the force constant $5.0$. Further simulation details are omitted since they are similar to those in~\cite{ae2024}. In total, $10^4$ configurations are obtained by recording every $100$ simulation steps. We exclude the hydrogen atoms and work with the coordinates of the $10$ non-hydrogen atoms in the system (see Figure~\ref{dipeptide_system}). In a final preparatory step, we apply the refinement technique in Appendix~\ref{Bunny_Spot_details} (see Eq.~\eqref{refine_ode}) to the recorded coordinates, so that the data in the dataset lives in the manifold $\mathcal{M}=\{x\in \mathbb{R}^{30}| \phi(x)=-70^\circ\}$ up to a small numerical error of order $10^{-5}$. 

Since $\mathcal{M}$ is unbounded, we 
adopt a nonzero function $b$ in our model to make sure that the Markov chain processes stay in bounded region. To this end, we choose a reference configuration
$x^\mathrm{ref}$ from the dataset and define the potential function 
\begin{equation}\label{potential-v}
V(x) = \frac{\kappa}{2} |R_x^*(x-w^*_x) - x^{\mathrm{ref}}|^2\,, \quad x\in \mathbb{R}^{30}\,,
\end{equation}
with $\kappa = 50$, where $R_x^*$, $w^*_x$ are the optimal rotation and the optimal translation that minimize the RMSD (see Eq.~\eqref{rmsd}). The function $b$ is defined as (the negative gradient of $V$ in full space)
\begin{equation}
b=-\nabla V(x) = -\kappa \left(R_x^*(x-w^*_x) - x^{\mathrm{ref}}\right),
\end{equation}
where the second equality follows by differentiating $V$ in \eqref{potential-v} and using the first order optimality equations satisfied by $R_x^*$ and $w^*_x$~(also see \cite{rmsd_quaternions}).

We also build our model to make sure that the generated distribution is $\mathrm{SE}(3)$-invariant (i.e. invariant under rotations and translations). For this, we rely on the theoretical results in \cite{xu2022geodiff} and in Appendix~\ref{Equivariant_NN}. One can check that $V(x)$ is $\mathrm{SE}(3)$-invariant and $b$ satisfies property \eqref{SE3_sym} in Appendix~\ref{Equivariant_NN}, that is, $b$ is equivariant under rotations and invariant under translations. This guarantees  that the prior distribution $p(x^{(N)})$, which we choose as the invariant distribution of the forward process, is $\mathrm{SE}(3)$-invariant as well. We still need to make sure that the transition densities of the reverse  Markov chain are $\mathrm{SE}(3)$-invariant. For this purpose, in the reverse process we set $s^{(k+1),\theta}(x) = (R^*_x)^\top f_\theta(R^*_x(x-w^*_x),\frac{(k+1)T}{N})$, where $f_{\theta}:\mathbb{R}^{30}\times \mathbb{R}\rightarrow\mathbb{R}^{30}$ is modeled by a single MLP with parameter $\theta$, and both $R^*_x$ and $b^*_x$ are computed by the Kabsch algorithm \cite{kabsch1976solution}. With this choice, $s^{(k+1),\theta}(x)$ satisfies property \eqref{SE3_sym} by Proposition \ref{equiv_nn} in Appendix~\ref{Equivariant_NN}, and the transition density of the reverse process is  $\mathrm{SE}(3)$-invariant by Proposition \ref{equiv_transition} in Appendix~\ref{Equivariant_NN}. Since the prior $p(x^{(N)})$ is also $\mathrm{SE}(3)$-invariant, we conclude that the learned distribution $p_{\theta}(x^{(0)})$ is $\mathrm{SE}(3)$-invariant~\cite{xu2022geodiff}. Compared to the commonly used equivariant networks~\cite{satorras2021n}, our network fits our experiment better thanks to its lower computational cost and reduced memory usage.

\subsection{Additional experimental results.}
This section presents additional experimental results, focusing on computation time and the non-convergence rate of trajectories. We first report the simulation time $T_{\mathrm{sim}}$ and training time $T_{\mathrm{train}}$, with the percentages of the total runtime $T_{\mathrm{total}}$ indicated in parentheses in Table \ref{runtime}. The trajectory update interval $l_{\mathrm{f}}$, the number of steps $N$, and the cost of solving Newton's equation jointly determine the simulation time $T_{\mathrm{sim}}$. Recall that the total complexity of Newton’s method is $\mathcal{O}(k_{\mathrm{iter}}(C_{\xi} +(n-d)^3))$, where $C_{\xi}$ denotes the computational cost of evaluating $\nabla \xi$. In our experiments, $C_{\xi}=\mathcal{O}(1)$ for most datasets, except for meshes where neural network evaluation is required. The codimension is typically $n-d = 1$, except for $\mathrm{SO}(10)$, where $n-d = 55$. When Newton's method converges, it requires at most $k_{\mathrm{iter}} = 3$ iterations.

Moreover, as discussed in Section \ref{subsec-projection-scheme}, \eqref{projection} may not yield solutions for certain vectors $v$, leading to the discarding of corresponding trajectories. Table~\ref{Newton} provides the failure rate of trajectory generation in our experiments, as well as the maximum number of iterations required for Newton's method in cases where it converges.

\begin{table}[ht]
\centering
\begin{tabular}{lccccccc}
\toprule
Datasets   & $l_{\mathrm{f}}$ & $N_{\mathrm{epoch}}$ & $T_{\mathrm{sim}}$ & $T_{\mathrm{train}}$  & $T_{\mathrm{total}}$  & $T_{\mathrm{epoch}}$ \\
\midrule
Bunny, $k=50$ & 100 & 2000 & 252(1.8\%) & 14120& 14372  &  7.06\\ 
Bunny, $k=100$ & 100 & 2000 & 142(1.6\%) & ~8718 & ~8860  & 4.36\\ 
Spot, $k=50$ & 100 & 2000 & 113(1.3\%)& ~8751  & ~8864 & 4.38\\ 
Spot, $k=100$ & 100& 2000 & ~70(1.3\%)& ~5189 & ~5259  & 2.59\\ 
\midrule
$\mathrm{SO}(10)$ & 100 & 2000 & 2426(11.2\%) & 19289 & 21715 & 9.64\\ 
Transformed-$\mathrm{SO}(10)$ & 100 & 2000 & 491(2.2\%) & 22276 & 22767 & 22.28 \\
\midrule
Hamiltonian & 1 & 4000 & 1938(79.6\%) & 496 & 2434 & 0.12\\
\midrule
Alanine dipeptide & 100 & 5000 & 159(1.1\%)& 14299 & 14458  & 2.86\\
\bottomrule
\end{tabular}
\caption{Detailed runtime metrics in our experiments. We report $T_{\mathrm{sim}}$, $T_{\mathrm{train}}$, and $T_{\mathrm{total}}$ as the time for path generation, time for training, and total runtime, respectively, with the percentages of the total runtime $T_{\mathrm{total}}$ shown in parentheses. The parameter $l_{\mathrm{f}}$ determines the frequency of trajectory updates. The final column $T_{\mathrm{epoch}}$ shows the training time per epoch, calculated as $T_{\mathrm{train}}/N_{\mathrm{epoch}}$. All time metrics are reported in seconds.}\label{runtime}
\end{table}

\begin{table}[htbp]
\centering
\begin{tabular}{lccccc}
\toprule
Datasets & $\sigma_{\max}$& $R_{\mathrm{fail\_fwd}}$ & $R_{\mathrm{fail\_bwd}}$ & $\mathrm{iter}_{\max}$ & $\mathrm{tol}$\\
\midrule
Bunny, $k=50$& 0.007 & 1.00\% &  0.82\% & 3 & 1e-4 \\ 
Bunny, $k=100$& 0.007 & 0.65\% & 0.55\% & 3& 1e-4 \\ 
Spot, $k=50$& 0.010 & 0.15\% & 0.25\% & 3 & 1e-4 \\ 
Spot, $k=100$& 0.010& 0.11\% & 0.10\% & 3 & 1e-4 \\ 
\midrule
$\mathrm{SO}(10)$ & 0.089 & 0.00\% & 0.00\%  & 3 & 1e-6\\ 
Transformed-$\mathrm{SO}(10)$ &  0.13 & 0.00\% & 0.00\% & 3 & 5e-6 \\
\midrule
Hamiltonian & 0.15 & 0.00\% & 0.00\% & 4 & 1e-5\\
\midrule
Alanine dipeptide & 0.022 & 0.01\% & 0.00\%  & 2& 1e-5 \\
\bottomrule
\end{tabular}
\caption{Failure rate of trajectory generation. $R_{\mathrm{fail\_fwd}}$ and $R_{\mathrm{fail\_bwd}}$ represent the percentages of discarded trajectories when sampling the forward and reverse process, respectively. $\sigma_{\max}$ denotes the maximum value of $(\sigma_k)_{0\leq k \leq N-1}$. $\mathrm{iter}_{\max}$ denotes the maximum number of iterations for Newton's method to converge (i.e. until the error is less than $\mathrm{tol}$).}\label{Newton}
\end{table}

\section{Theoretical results on neural networks for molecular systems}\label{Equivariant_NN}

In this section, we present theoretical results for the neural network architecture we employed in studying alanine dipeptide. 

Assume that the system consists of $M$ atoms, where the coordinates of the $i$-th atom are denoted by $\boldsymbol{x}_i\in \mathbb{R}^{3}$, for $i=1,2,\dots, M$. Let $x\in \mathbb{R}^{3M}$ be the vector consisting of all the coordinates $\boldsymbol{x}_1\,\boldsymbol{x}_2,\dots, \boldsymbol{x}_M\in \mathbb{R}^{3}$. For simplicity, given a rotation matrix $R\in \mathrm{SO}(3)$ and a translation vector $w\in\mathbb{R}^3$, we use the conventional notation $Rx+w$ to denote the vector in $\mathbb{R}^{3M}$ that consists of the transformed coordinates $R\boldsymbol{x}_1+w, R\boldsymbol{x}_2+w, \dots, R\boldsymbol{x}_M+w\in \mathbb{R}^3$. We say that a function $f$ defined in $\mathbb{R}^{3M}$ is $\mathrm{SE}(3)$-invariant, if $f(Rx+w) = f(x)$, for all  $R\in \mathrm{SO}(3)$, $w\in\mathbb{R}^3$, and for all $x\in \mathbb{R}^{3M}$. We say that function $f:\mathbb{R}^{3M}\rightarrow \mathbb{R}^{3M}$ possesses property \eqref{SE3_sym}, if it is both equivariant under rotations and invariant under translations, i.e.\ 
\begin{equation}\label{SE3_sym}
    f(Rx+w) = Rf(x),\mbox{ for all }R\in \mathrm{SO}(3), w\in\mathbb{R}^3, \mbox{ and all }x\in \mathbb{R}^{3M}.
\end{equation}
Assume that a configuration $x^\textrm{ref}$ is chosen as reference. Given $x$, the optimal rotation matrix and the optimal translation vector, which minimize the RMSD 
\begin{equation}
\mathrm{RMSD}(x;x^\textrm{ref}) = \Big(\frac{1}{M}|R (x - w) - x^\textrm{ref}|^2\Big)^{\frac{1}{2}}
\label{rmsd}
\end{equation}
from the reference $x^\textrm{ref}$, are denoted by $R^*_x$ and $w^*_x$, respectively.

Proposition \ref{equiv_nn} characterizes functions that are both equivariant under rotations and invariant under translations. Proposition \ref{equiv_transition} guarantees the $\mathrm{SE}(3)$-invariance of the transition densities of our diffusion model.

\begin{proposition}\label{equiv_nn}
The following two claims are equivalent.
\begin{itemize}
\item
Function $s:\mathbb{R}^{3M}\rightarrow\mathbb{R}^{3M}$  possesses property \eqref{SE3_sym}. 
\item
There is a function $f:\mathbb{R}^{3M}\rightarrow\mathbb{R}^{3M}$, such that $s(x) = (R^*_x)^\top f(R^*_x(x-w^*_x))$, for all $x\in \mathbb{R}^{3M}$.
\end{itemize}
\label{prop-equivariant}
\end{proposition}
\begin{proof}
It is straightforward to verify that the first claim implies the second claim. In fact, setting $R=R_x^*$, $w=-R^*_xw_x^*$, and using the identity $R^\top R=I_3$, we obtain from the first claim that $s(x)=(R_x^*)^\top s(R^*_x(x-w^*_x))$. Hence, the second claim holds with $f=s$. To show that the second claim also implies the first one, we use the fact that the optimal rotation $R_{Rx+w}^*$ and the optimal translation $w_{Rx+w}^*$, which minimize the RMSD of the state $Rx+w$ from the reference $x^{\mathrm{ref}}$, are given by $R_{Rx+w}^*=R_{x}^*R^\top$ and $w_{Rx+w}^*=Rw^*_x+w$, respectively. This fact can be directly checked using \eqref{rmsd}. In particular, we have 
\begin{equation*}
R^*_{Rx+w} (Rx+w - w^*_{Rx+w}) =
R^*_x(x - w^*_x)\,.
\end{equation*}
Therefore, for the function $s$ defined in the second claim, we can compute, for any $R\in \mathrm{SO}(3)$, $w\in\mathbb{R}^3$, and any $x\in \mathbb{R}^{3M}$,
\begin{align*}
s(Rx + w)=&(R^*_{Rx+w})^\top f\big(R^*_{Rx+w}(Rx+w - w^*_{Rx+w})\big) \\
=& R (R^*_x)^\top f(R^*_x (x -w^*_x))\\
=& R s(x)\,,
\end{align*}
which shows the first claim. \hfill
\end{proof}

\begin{proposition}\label{equiv_transition}
Assume that $\xi$ is $\mathrm{SE}(3)$-invariant and $b$ possesses property \eqref{SE3_sym}. Then, the transition density \eqref{trans-forward-projection} of the forward process is $\mathrm{SE}(3)$-invariant. Further assume that the function $s^{(k+1),\theta}$ possesses property \eqref{SE3_sym} for $0 \le k\le N-1$. Then, the transition density \eqref{trans-backward-projection} of the reverse process is also $\mathrm{SE}(3)$-invariant.
\end{proposition}

\begin{proof}
We consider the transition density \eqref{trans-forward-projection}. Recall that $U_x\in \mathbb{R}^{n\times d}$ is a matrix whose columns form an orthonormal basis of $T_x\mathcal{M}$. Since $\xi$ is $\mathrm{SE}(3)$-invariant, we have $\xi(Rx+w)=\xi(x)$, for all rotations $R$ and translation vectors $w$, which implies that $Rx+w\in \mathcal{M}$, if and only if $x \in\mathcal{M}$. Differentiating the identity $\xi(Rx+w)=\xi(x)$ with respect to $x$, we obtain the relation $\nabla\xi(Rx+w) = R\nabla\xi(x)$, from which we see that $U_{Rx+w}$ can be chosen such that  $U_{Rx+w}=RU_{x}$. For the orthogonal projection matrix $P$ in \eqref{exp-pij}, using the identity $R^\top R = I_3$, we can compute 
\begin{align*}
P(Rx+w)=&I_n - \nabla\xi(Rx+w) \big(\nabla\xi(Rx+w)^\top\nabla\xi(Rx+w)\big)^{-1} \nabla\xi(Rx+w)^\top \\
=& I_n - R\nabla\xi(x) \big(\nabla\xi(x)^\top\nabla\xi(x)\big)^{-1} \nabla\xi(x)^\top R^\top \\
=& R P(x) R^\top\,.
\end{align*}
Moreover, since both $b$ and $\nabla \xi$ satisfy the property \eqref{SE3_sym}, we also have $\epsilon_{Rx^{(k)}+w}^{(\sigma_k)}= \epsilon_{x^{(k)}}^{(\sigma_k)}$ (i.e.\ the probabilities of having no solution are the same). Therefore, for the transition density \eqref{trans-forward-projection}, we can derive, for any $R\in \mathrm{SO}(3)$ and $w\in\mathbb{R}^3$,
\begin{align*}
&q(Rx^{(k+1)}+w\,|\,Rx^{(k)}+w) \\
=& (2\pi\sigma_k^2)^{-\frac{d}{2}} \big(1-\epsilon_{Rx^{(k)}+w}^{(\sigma_k)}\big)^{-1}
|\det (U_{Rx^{(k)}+w}^\top U_{Rx^{(k+1)}+w})|
\\
& \quad\quad \cdot \exp(-\frac{\big|P(Rx^{(k)}+b)\big(Rx^{(k+1)}-Rx^{(k)}-\sigma_k^2 b(Rx^{(k)}+w)\big)\big|^2}{2\sigma_k^2}) \\
  =& (2\pi\sigma_k^2)^{-\frac{d}{2}} \big(1-\epsilon_{x^{(k)}}^{(\sigma_k)}\big)^{-1} \left|\det (U_{x^{(k)}}^\top R^\top R U_{x^{(k+1)}})\right| \\
  & \quad\quad \cdot
  \exp(-\frac{\big|RP(x^{(k)})R^\top\big(Rx^{(k+1)}-Rx^{(k)}-\sigma_k^2 Rb(x^{(k)})\big)\big|^2}{2\sigma_k^2})   \\
  =& (2\pi\sigma_k^2)^{-\frac{d}{2}} \big(1-\epsilon_{x^{(k)}}^{(\sigma_k)}\big)^{-1} \left|\det (U_{x^{(k)}}^\top U_{x^{(k+1)}})\right| \exp(-\frac{\big|P(x^{(k)})\big(x^{(k+1)}-x^{(k)}-\sigma_k^2 b(x^{(k)})\big)\big|^2}{2\sigma_k^2} ) \\
  =& q(x^{(k+1)}\,|\,x^{(k)})\,,
\end{align*}
which shows the $\mathrm{SE}(3)$-invariance of the transition density of the forward process. The invariance of the transition density of the reverse process in \eqref{trans-backward-projection} can be proved using the same argument, assuming that $s^{(k+1),\theta}$ satisfies the relation $s^{(k+1),\theta}(Rx+w)=Rs^{(k+1),\theta}(x)$. \hfill
\end{proof}

\section{Ablation study}

\begin{table}[htbp]
\centering
\begin{tabular}{lcccccc}
\toprule
$N$ & $\sigma$ & $R_{\mathrm{fail\_fwd}}$ & $T_{\mathrm{sim}}$ & $T_{\mathrm{train}}$ & JS & NLL \\
\midrule
50 & 2.45e-2 & 2.86\% & 19.2(3.2\%) & 589.3 & 4.382e-01\scriptsize{$\pm$3.81e-03} & 0.77\scriptsize{$\pm$0.01} \\
100 & 1.73e-2 & 1.23\% & 36.9(2.7\%) & 1313.1 & 4.314e-01\scriptsize{$\pm$2.18e-03}  & 0.77\scriptsize{$\pm$0.01} \\
200 & 1.22e-2 & 0.34\% & 54.0(2.0\%) & 2673.0 & 4.312e-01\scriptsize{$\pm$2.66e-03}  & 0.77\scriptsize{$\pm$0.01}\\
300 & 1.00e-2 & 0.11\% & 66.4(1.6\%) & 4186.2 & 4.298e-01\scriptsize{$\pm$2.89e-03}  & 0.77\scriptsize{$\pm$0.01}\\
400 & 0.87e-2 & 0.05\% & 84.8(1.2\%) & 7297.1 & 4.300e-01\scriptsize{$\pm$1.68e-03}  & 0.78\scriptsize{$\pm$0.01} \\
\bottomrule
\end{tabular}
\caption{Ablation study of the number of steps $N$ on the Spot the Cow dataset with $k=100$. Here, $\sigma$ denotes the step size and $R_{\mathrm{fail\_fwd}}$ represents the proportion of failed forward trajectories. $T_{\mathrm{sim}}$ and $T_{\mathrm{train}}$ are the simulation and training times (see Table~\ref{runtime}). JS denotes the Jensen-Shannon distance on mesh face histograms between the ground truth and generated distributions.}\label{spot100_N_aba}
\end{table}

\begin{table}[htbp]
\centering
\begin{tabular}{lcccccc}
\toprule
$N$ & $\sigma$ & $R_{\mathrm{fail\_fwd}}$ & $T_{\mathrm{sim}}$ & $T_{\mathrm{train}}$ & $\mathcal{W}_1$ \\
\midrule
50 & 0.26 & 0\% & 823.2(81.1\%) & 191.5 & 3.152e-01\scriptsize{$\pm$1.96e-03} \\
100 & 0.18 & 0\% & 1429.0(85.0\%) & 253.0 & 3.117e-01\scriptsize{$\pm$8.09e-04} \\
150 & 0.15 & 0\% & 2035.8(84.3\%) & 378.4 & 3.098e-01\scriptsize{$\pm$3.48e-03} \\
200 & 0.13 & 0\% & 2717.9(84.9\%) & 482.9 & 3.093e-01\scriptsize{$\pm$2.66e-03} \\
\bottomrule
\end{tabular}
\caption{Ablation study of the number of steps $N$ on conserved Hamiltonian surfaces. $\mathcal{W}_1$ denotes the 1-Wasserstein distance in coordinate space $(q_1,q_2,\dots,q_n)$ between the ground truth and generated distributions.}\label{Hamilton_N_aba}
\end{table}

Table \ref{spot100_N_aba} presents the computational cost and performance analysis for the Spot the Cow dataset with $k=100$ under varying $N$. We observe that, as $N$ increases, the Newton method's failure rate decreases, and both trajectory simulation time and training time increase accordingly. To better assess the generation quality, we compare the Jensen-Shannon distance between the ground truth and generated distributions, as the NLL metric fails to reveal meaningful differences. The results indicate that $N=50$ yields relatively poor performance, while other values of $N$ show comparable results with marginal differences. 

Table \ref{Hamilton_N_aba} presents the ablation study on conserved Hamiltonian surfaces. The results demonstrate that as $N$ increases, the computational cost grows while the accuracy improves accordingly.

The performance of our method is not highly sensitive to the number of steps $N$, as long as $N$ is not too small; the computational cost scales positively correlated with $N$. Thus, a viable strategy is to choose $N$ by balancing the Newton method's failure rate against computational overhead, while maintaining satisfactory generation quality.


\bibliographystyle{myabbrv}
\bibliography{ref}


\end{document}